\documentclass{article}
\usepackage[a4paper, total={6in, 8in}]{geometry}

\usepackage{amsthm}
\usepackage{amssymb}
\usepackage{amsmath}
\usepackage{mathrsfs}
\usepackage{natbib}
\usepackage{selectp}
\usepackage{enumitem}
\usepackage[normalem]{ulem}
\usepackage{hyperref}
\usepackage{authblk}
\usepackage{epsfig}
\usepackage{graphicx}

\hypersetup{
    bookmarks=true,         
    unicode=false,          
    pdftoolbar=true,        
    pdfmenubar=true,        
    pdffitwindow=false,     
    pdfstartview={FitH},    
    pdftitle={Single Trajectory Nonparametric Learning of Nonlinear Dynamics},    
    pdfauthor={Ingvar Ziemann, Henrik Sandberg, Nikolai Matni},     
    pdfsubject={Machine Learning},   
    pdfcreator={Ingvar Ziemann},   
    pdfproducer={Ingvar Ziemann}, 
    pdfkeywords={keyword1, key2, key3}, 
    pdfnewwindow=true,      
    colorlinks=true,       
    linkcolor=blue,          
    citecolor=blue,        
    filecolor=magenta,      
    urlcolor=red,           
    breaklinks=true
}

\usepackage{setspace}
\newtheorem{theorem}{Theorem} 
\newtheorem{proposition}{Proposition} 

\newtheorem{lemma}{Lemma}
\newtheorem{example}{Example}

\newcommand{\e}{\varepsilon}

\newcommand{\E}{\mathbf{E}}

\DeclareMathOperator{\argmin}{argmin}

\DeclareMathOperator{\polylog}{polylog}

\DeclareMathOperator{\tr}{tr}

\title{Single Trajectory Nonparametric Learning of Nonlinear Dynamics}
\date{}
\author[1]{Ingvar Ziemann}
\author[1]{Henrik Sandberg}
\author[2]{Nikolai Matni}
\affil[1]{Division of Decision and Control Systems,  KTH Royal Institute of Technology}
\affil[2]{Department of Electrical and Systems Engineering, University of Pennsylvania}
\begin{document}



\maketitle
\begin{abstract}
Given a single trajectory of a dynamical system, we analyze the performance of the nonparametric least squares estimator (LSE). More precisely, we give nonasymptotic expected $l^2$-distance bounds between the LSE and the true regression function, where expectation is evaluated on a fresh, counterfactual, trajectory. We leverage recently developed information-theoretic methods to establish the optimality of the LSE for nonparametric hypotheses classes  in terms of supremum norm metric entropy and a subgaussian parameter. Next, we relate this subgaussian parameter to the stability of the underlying process using notions from dynamical systems theory. When combined, these developments lead to rate-optimal error bounds that scale as $T^{-1/(2+q)}$ for suitably stable processes and hypothesis classes with metric entropy growth of order $\delta^{-q}$. Here,  $T$ is the length of the observed trajectory, $\delta \in \mathbb{R}_+$ is the packing granularity and $q\in (0,2)$ is a complexity term. Finally, we specialize our results to a number of scenarios of practical interest, such as Lipschitz dynamics, generalized linear models, and dynamics described by functions in certain classes of Reproducing Kernel Hilbert Spaces (RKHS).


\end{abstract}



\thispagestyle{plain}
\section{Introduction}

Consider a time-series model of the form
\begin{align}
\label{eq:ds}
y_{t} &= f_\star(x_t) +w_t,  \:\: \: t=0,\dots,T-1.
\end{align}
where $f_\star$ is unknown but belongs to some known function class $\mathbb{F}$. Suppose a learner is given access to samples $(x_0,\dots,x_{T-1},y_0,\dots,y_{T-1})$, corrupted by noise $(w_0,\dots,w_{T-1})$, from a single trajectory generated by model (\ref{eq:ds}). In this work, we show that the nonparametric least squares estimator (LSE) converges to the ground truth $f_\star$ at the minimax optimal rate. In the i.i.d. setting, in which each observation from the model (\ref{eq:ds}) is drawn independently at random, the optimal rate is $T^{-1/2}$ for parametric models.  In the nonparametric setting, this rate degrades gracefully to $T^{-1/(2+q)}$ for models with metric entropy scaling as $\delta^{-q}, q\in (0,2)$ \citep{tsybakov2009introduction}. In this paper, we show that these rates can be matched for a class of more general time-series models of the form \eqref{eq:ds}.   We note in particular that by setting $y_t = x_{t+1}$ in model (\ref{eq:ds}), we recover the nonlinear stochastic dynamical system
\begin{align}
\label{eq:ards}
x_{t+1}=f_\star(x_t)+w_t.
\end{align}

Estimation of models (\ref{eq:ds}) and (\ref{eq:ards}) remains relatively poorly understood when the data is not i.i.d., with existing results being limited to when the function $f_{\star}$ is known to belong to certain parametric classes. In terms of parameter recovery, the LSE converges at a rate of $T^{-1/2}$ for stable linear autoregressive systems $f_\star(x_t) = A_\star x_t$ \citep{simchowitz2018learning, sarkar2019near, jedra2020finite}. The same rate can also be achieved for linear systems with more general input-output behavior \citep{oymak2019non, tsiamis2019finite}. Moving to nonlinear models, recursive and gradient type algorithms can be shown to converge at a rate of $T^{-1/2}$ for the generalized linear model  $f_\star(x_t) = \phi(A_\star x_t)$, where $\phi$ is a known Lipschitz link function \citep{foster2020learning, sattar2020non, jain2021near}.  In this paper, we significantly generalize these results and provide rate-optimal error bounds for nonparametric function classes in terms of their metric entropies. Our approach leverages recently developed information-theoretic tools \citep{russo2019much, NIPS2017_ad71c82b} and the notion of offset complexity \citep{rakhlin2014online, liang2015learning}.

\paragraph{Problem Formulation}

The dynamics (\ref{eq:ds}) evolve on two subsets of Euclidean space: $\mathsf{X} \subset \mathbb{R}^{d_x}$ with $x_t \in \mathsf{X}$ and $\mathsf{Y}\subset \mathbb{R}^{d_y}$ with $y_t \in \mathsf{Y}$. We assume that there exists an increasing sequence of $\sigma$-fields $\{\mathcal{F}_t\}_{t\in\mathbb{Z}_{\geq -1}}$ such that each $x_t$ is $\mathcal{F}_{t-1}$-measurable, each $w_t$ is $\mathcal{F}_t$-measurable and $\E[w_t| \mathcal{F}_{t-1}]=0$. In other words, $\{w_t\}_{t\in\mathbb{Z}_{\geq 0}}$ is a martingale difference sequence with respect to the filtration $\{\mathcal{F}_t\}_{t\in\mathbb{Z}_{\geq 0}}$ and $\{x_t\}_{t\in\mathbb{Z}_{\geq 0}}$ is adapted to $\{\mathcal{F}_{t-1}\}_{t\in \mathbb{Z}_{\geq -1}}$. Furthermore, we assume that each $w_t$ is conditionally $\sigma^2_w$-subgaussian given $\mathcal{F}_{t-1}$. We denote  by ${P}_\star$ the joint distribution of $Z=(x_0,\dots,x_{T-1},y_0,\dots,y_{T-1})$, with $x_t$ and $y_t$ as in system (\ref{eq:ds}). We assume that $f_\star$ is unknown but that it belongs to a known metric space $(\mathbb{F},\rho)$ with $\rho(f,g)=\sup_{x\in \mathsf{X}} \|f(x)-g(x)\|_2$. See Section~\ref{sec:prelnot} for further preliminaries.

Given this, the learning task is to produce an estimate $\hat f$ of the model $f_\star$, which is evaluated in terms of the expected Euclidean $2$-norm error:
\begin{align}
\label{eq:defL}
 \E \|\hat f(\xi)-f_\star(\xi)\|_2.
\end{align}
The expectation (\ref{eq:defL}) is computed with respect to the randomness in the algorithm and the random variable $\xi\sim \nu$ which is independent of all other randomness; here $\nu$ is the uniform mixture over $(x_0,\dots,x_{T-1})$. That is, $\xi$ has the same distribution as random variable $x_\tau$, where the index $\tau$ is drawn uniformly at random over $\{0,\dots,T-1\}$ and is independent of $(x_0,\dots,x_{T-1})$. If the process (\ref{eq:ds}) is stationary with invariant measure $\nu$, this reduces to the assumption that $\xi$ is drawn from the invariant measure of the process. 

In the sequel, we analyze the nonparametric least squares estimator (LSE) $\hat f$ of $f_\star$, given below. Namely, we assume that the learner can compute
\begin{align}
\label{df:PEM}
\hat f \in \argmin_{f\in \mathbb{F}}\left\{ \frac{1}{T} \sum_{t=0}^{T-1} \|y_t-f(x_t)\|_2^2  \right\}.
\end{align}

\subsection{Contributions}

We derive error bounds for learning nonlinear dynamics  (\ref{eq:ards}) and the more general time-series model (\ref{eq:ds}).  Theorem~\ref{thm:ezpzthm} provides  bounds on $\E \|\hat f(\xi)-f_\star(\xi)\|_2$ for the LSE (\ref{df:PEM}). These bounds depend only on the metric entropy of the hypothesis class, $\log \mathcal{N}(\mathbb{F},\|\cdot\|_\infty,\delta)$, the noise level $\sigma_w^2$, and a further variance proxy $\sigma_T^2(\mathbb{F},P_\star)$, measuring the spatiotemporal spread of the covariates relative to the function class $\mathbb{F}$. Informally, our main result, Theorem~\ref{thm:ezpzthm}, states that
\begin{multline}
\label{informalinequality}
    \E \| \hat f(\xi)- f_\star(\xi) \|_2 \lesssim \left(\frac{\textnormal{dimensional factors} \times \textnormal{measurement noise}}{\textnormal{trajectory length} }\right)^{1/(2+\textnormal{complexity term})}\\
    +\left(\frac{\textnormal{dimensional factors} \times \textnormal{covariates spread}}{\textnormal{trajectory length} }\right)^{1/(2+\textnormal{complexity term})}
\end{multline}
where the dimensional factors and the complexity term depends on the scaling of the metric entropy $\log \mathcal{N}(\mathbb{F},\|\cdot\|_\infty,\delta)$ for small $\delta>0$. Our bounds exhibit optimal scaling in terms of interaction between trajectory length and function class complexity in that they agree with  known minimax optimal rates for the i.i.d. setting \citep{tsybakov2009introduction}. To the best of our knowledge, these are the first such bounds that are applicable to large nonparametric function classes in the temporally correlated (non-i.i.d.) setting.

Arriving  at bounds of the form (\ref{informalinequality}) for temporally correlated data is challenging, as the symmetrization technique typically used to analyze both generalization and training error for regression cannot be applied.  Instead, we leverage information-theoretic decoupling arguments introduced by \cite{russo2019much} and \cite{NIPS2017_ad71c82b} to reduce the analysis of the error (\ref{eq:defL}) to bounding an in-sample prediction error (training error) and a term measuring the dependence between the sample and the algorithm (generalization error). We analyze the first term using the offset basic inequality \citep{rakhlin2014online, liang2015learning}. Crucially, this leads to a simplified localization argument that is amenable to modification for correlated data and allows us to obtain fast rates. The second term of the decoupling estimate is bounded by the mutual information between the algorithm and the sample, which we control via metric entropy and discretization. Moreover, the scale of this term is controlled by the spatiotemporal variance proxy, $\sigma_T^2(\mathbb{F},P_\star)$. 

In the information-theoretic generalization bounds literature \citep{NIPS2017_ad71c82b}, the term $\sigma_T^2(\mathbb{F},P_\star)$ is referred to as the subgaussian parameter of the loss function. Here, given the form (\ref{eq:ds}) of the data generating process and our choice of performance metric, the variance proxy $\sigma_T^2(\mathbb{F},P_\star)$ admits a more direct interpretation in terms of the stability of the process (\ref{eq:ards}). Namely, by an explicit stability argument we show that $\sigma_T^2(\mathbb{F},P_\star) \lesssim 1 /((1-L_\star)^2T)$  whenever the autoregressive system (\ref{eq:ards}) is $L_\star$-contractive (Proposition~\ref{prop:cig}). We also show that $\sigma_T^2(\mathbb{F},P_\star) \lesssim 1/T$ holds more generally whenever the covariates of the process  (\ref{eq:ds}) form a Markov chain with finite mixing time (Proposition~\ref{prop:mig}). Given the recent line of work emphasizing the role of control-theoretic stability for learning in dynamical systems \citep{foster2020learning, boffi2021regret, tu2021sample}, we believe that this is an attractive construction that may have wider applicability\footnote{See Appendix~\ref{subsec:genbounds} for a discussion on how our results apply to generalization bounds for Lipschitz losses.}.

\subsection{Further Related Work}

Estimation of models of the form (\ref{eq:ds}) and (\ref{eq:ards}) has a rich history in statistics and system identification \citep{lennart1999system}. Preceding the recent body of work mentioned in the introduction, asymptotically optimal rates for linear stochastic models have been available for some time \citep{mann1943statistical, lai1982least}. Similarly, there is a well-established theory of rate-optimal identification for nonlinear parametric models under various identifiability-type conditions, both in the i.i.d. setting \citep{van2000asymptotic} and under more general assumptions \citep{le2012asymptotic}.

Perhaps the main motivator for the recent line of work emphasizing nonasymptotic estimation bounds is that these bounds are applicable downstream in control and reinforcement learning pipelines. Regression estimates for linear stochastic systems have been key to understanding both online and offline reinforcement learning in the linear quadratic regulator \citep{dean2020sample, mania2019certainty} and can be shown to lead to optimal regret rates \citep{simchowitz2020naive, ziemann2022regret}. Extending our understanding of the interaction between learning and control beyond linear-in-the-parameters models \citep{kakade2020information, boffi2021regret, lale2021model} inevitably requires new analyses of learning in dynamical systems. This also motivates the present work in that we provide nonasymptotic and counterfactual control of the LSE's estimation error for more general nonlinear and nonparametric models. 

Another related field is that of general statistical learning for dependent data, see \cite{agarwal2012generalization}, \cite{kuznetsov2017generalization} and the references therein. These works provide generalization bounds for general loss functions and  $\beta$-mixing processes $(x_t,y_t)$. The assumption of $\beta$-mixing processes has also previously been exploited in parametric identification by \cite{vidyasagar2006learning}.  By contrast, our emphasis on regression over general learning is motivated by downstream applications in learning-enabled control, where one first learns a model used to design a controller. Necessarily then, this work builds on a rich line of work in nonparametric regression for the i.i.d. setting, see chapters 13 and 14 of \cite{wainwright2019high} and the references therein. Although there has been some work on the dependent setting, the error in this line of work is typically computed with respect to the design points which is not suitable for the counterfactual reasoning that is key to control \citep{baraud2001adaptive}. 

We also draw inspiration from the recent line of work on information-theoretic generalization bounds \citep{NIPS2017_ad71c82b, russo2019much, bu2020tightening}. There are interesting  refinements and variations of this theory using for instance conditional mutual information \citep{steinke2020reasoning}, or Wasserstein distance \citep{galvez2021tighter}. However, these more recent bounds rely more explicitly on the  tenzorization properties of information measures under i.i.d. data than the earlier work of \cite{russo2019much} and \cite{NIPS2017_ad71c82b}, and so are not directly amenable to the single trajectory setting. We also note that information-theoretic generalization bounds have previously found other applications, such as in the analysis of  stochastic gradient descent \citep{pmlr-v134-neu21a}.

\subsection{Preliminaries and Notation}
\label{sec:prelnot}
All logarithms used in this paper are base $e$. For two probability measures $\mathbf{P}$ and $\mathbf{Q}$ we denote by $D(\mathbf{P}\|\mathbf{Q}) =\int \log \frac{d\mathbf{P}}{d\mathbf{Q}} d\mathbf{P} $ their Kullback-Leibler divergence and their total variation distance by $d_{TV}(\mathbf{P},\mathbf{Q}) = \frac{1}{2}\int |d\mathbf{P}-d\mathbf{Q}|$. For a random variable $X$ we denote its law by $P_X$, that is $X\sim P_X$. For two random variables $X$ and $Y$, $(X,Y) \sim P_{X,Y}$, we denote by $I(X;Y) = D(P_{X,Y} \| P_{X \otimes Y})$ their mutual information, where $P_{X \otimes Y}=P_X \otimes P_Y$ denotes the product measure of the marginal distributions of $X$ and $Y$. If $X$ is a random variable over a finite alphabet $\{1,\dots,M\}$, we denote by $H(X)$ its Shannon entropy which is given by $H(X) =- \sum_{j=1}^M [\log P_X(j) ]P_X(j)$. Generic expectation (integration with respect to all randomness) is denoted by $\mathbf{E}$. A random variable $X$ taking values in $\mathbb{R}^d$ is said to be $\sigma^2$-subgaussian if $\E \exp \lambda \langle X-\E X,s\rangle  \leq \exp \lambda^2 \sigma^2/2$ for all $s \in \mathbb{S}^{d-1}$, where $ \mathbb{S}^{d-1}\subset \mathbb{R}^d$ is the unit sphere and $\langle \cdot, \cdot \rangle$ is the standard Euclidean inner product. This  extends to conditional subgaussianity via conditional expectation, $\E[ \cdot | \mathcal{F}]$, with respect to a $\sigma$-field $\mathcal{F}$, if the same holds with expectation $\E$ exchanged for $\E[ \cdot | \mathcal{F}]$.

Let $(\mathbb{F},\rho)$ be a metric space. We define its $\delta$-covering number $\mathcal{N}(\mathbb{F},\rho,\delta)$ as the cardinality of the smallest $\delta$-cover of $\mathbb{F}$ in the metric $\rho$. In this case we say that $\mathbb{F}$  has metric entropy $\log \mathcal{N}(\mathbb{F},\rho,\delta)$ where $\mathcal{N}(\mathbb{F},\rho,\delta)$ is the $\delta$-covering number of $(\mathbb{F},\rho)$. If no such covering exists we write $\log \mathcal{N}(\mathbb{F},\rho,\delta)=\infty$.  Recall that $P_\star$ denotes the distribution of $Z = (x_0,\dots,x_{T-1},y_0,\dots y_{T-1})$ under the dynamics (\ref{eq:ds}). If
\begin{align*}
\sigma_T^2(\mathbb{F},P_\star) \triangleq \inf\left\{ \sigma^2 : \frac{1}{T} \sum_{t=0}^{T-1} \| f(x_t)-g(x_t) \|_2 \textnormal{ is $\sigma^2$-subgaussian under ${P}_\star$} \textnormal{ for all } f,g \in \mathbb{F} \right\}
\end{align*}
is finite, we say that the space $\mathbb{F}$ is $\sigma_T^2(\mathbb{F},P_\star)$-subgaussian with respect to the dynamics (\ref{eq:ds}). We shall make the assumption that $\mathbb{F}$ is $\sigma_T^2(\mathbb{F},P_\star)$-subgaussian. Importantly, this implies that all the centered functions $f-f_\star$ are subgaussian.  Observe that this a property  defined both in terms of the space $\mathbb{F}$ and the system (\ref{eq:ds}).  As the variable $f-f_\star$ will appear frequently throughout the text,  it will be convenient to define $\mathbb{F}_\star \triangleq \mathbb{F}-f_\star$. Further, it will be useful for purposes of analysis to quantize the estimate $\hat f$ given by the LSE (\ref{df:PEM}). For $\mathbb{F}_\delta=\{f_1,\dots, f_M\}$ an optimal $\delta$-covering of $\mathbb{F}$, we define $\hat f_\delta$ to be a quantization of the LSE as follows
\begin{align}
\label{df:PEMdisc}
\hat f_\delta \in \argmin_{f\in \mathbb{F}_\delta} \rho(f,\hat f).
\end{align}

The following shorthand notation will also be used to ease the exposition: we write $a_t \lesssim b_t$ if there exists a universal constant $C$ such that $a_t \leq C b_t$ for every $t \geq t_0$ and some $t_0 \in \mathbb{N}$. If $a_t \lesssim b_t$ and $b_t \lesssim a_t$ we write $a_t \asymp b_t$. The same convention applies for functions of $\delta$, the parameter of metric entropy, instead of $t$, but in the small $\delta$ regime (typically $\delta$ will be in inverse proportion to some increasing function of $t$).

\section{Results}
\label{sec:mainsec}
Our main result is an error bound that controls the distance between the estimate $\hat f$, defined by the nonparametric LSE (\ref{df:PEM}), and the ground truth $f_\star$ in terms of a fresh sample $\xi$ drawn independently of the algorithm  from the mixture distribution over the samples $(x_0,\dots,x_{T-1})$. 

\begin{theorem}
\label{thm:ezpzthm}
Fix a metric space $(\mathbb{F},\rho)$ with $\rho(f,g) = \sup_{x\in \mathsf{X}} \|f(x) -g(x)\|_2$  and assume that $f_\star \in \mathbb{F}$. Then for any $\delta>0, \gamma>0$ and $\alpha \in [0,\gamma]$, the LSE (\ref{df:PEM}), satisfies
\begin{equation}
\begin{aligned}
\label{eq:metentthm}
\E \|\hat f(\xi)-f_\star(\xi)\|_2  &\leq\sqrt{\frac{8\sigma^2_w\log \mathcal{N}(\mathbb{F},\rho,\gamma)}{T}+128\alpha\sigma_w\sqrt{d_y}+64\frac{\sigma_w } {\sqrt{T}} \int_\alpha^\gamma \sqrt{\log \mathcal{N}(\mathbb{F},\rho,s)}ds  }\\&+3\delta +\sqrt{2\sigma_T^2(\mathbb{F},P_\star)\log \mathcal{N}(\mathbb{F},\rho,\delta)}.
\end{aligned}
\end{equation}
\end{theorem}

To establish inequality (\ref{eq:metentthm}), we rely on an information-theoretic decoupling argument given in Proposition~\ref{genpredprop}. Informally,  Proposition~\ref{genpredprop} allows us to decompose the error as
\begin{align*}
    \E \|\hat f(\xi)-f_\star(\xi)\|_2 \lesssim \textnormal{training error}(\gamma,\alpha,T) + \textnormal{generalization error}(\delta,\sigma_T^2(\mathbb{F},P_\star)).
\end{align*}
The $\textnormal{training error}$ term above is controlled by the offset basic inequality (Lemma~\ref{lemma:basicineq}). Discretizing and proceeding through chaining yields a maximal inequality with $\alpha$ and $\gamma$ as trade-off parameters. The $\textnormal{generalization error}$ term above is defined in terms of the discretization parameter $\delta$, which controls the discrepancy between the quantized model $\hat f_\delta$ (defined in (\ref{df:PEMdisc})) and the LSE $\hat f$. The quantized model $\hat f_\delta$, used solely in the proof, ``generalizes well'' since it belongs to a finite hypothesis class by construction.

While we typically set $\alpha=0$ in Theorem~\ref{thm:ezpzthm}, the optimal choices of $\delta$ and $\gamma$ depend on a critical balance: to arrive at an optimal bound we must balance the complexity of the hypothesis class, $\mathbb{F}$, through its metric entropy, with statistical properties of the model (\ref{eq:ds}), such as the sampling length $T$, the noise amplitude $\sigma_w^2$ and the variance proxy $\sigma_T^2(\mathbb{F},P_\star)$.  The full proof of Theorem~\ref{thm:ezpzthm} can be found in Appendix~\ref{sec:mainproofsec} and a more detailed outline is given in Section~\ref{sec:proofoutline}.

To make the consequences of Theorem~\ref{thm:ezpzthm} more explicit, we consider two different complexity regimes for $\mathbb{F}$. If there exist $p,q\in \mathbb{R}_+$ such that
\begin{align}
\label{eq:nonparametricregime}
\log \mathcal{N}(\mathbb{F},\|\cdot\|_\infty,\delta) \lesssim p \left(\frac{1}{\delta}\right)^q
\end{align}
we are in the nonparametric regime. If instead there exist $p,c\in \mathbb{R}_+$ such that \begin{align}
\label{eq:parametricregime}
\log \mathcal{N}(\mathbb{F},\|\cdot\|_\infty,\delta) \lesssim p \log \left(1+\frac{c}{\delta}\right)
\end{align}
we are in the parametric regime. Concrete examples of processes satisfying conditions (\ref{eq:nonparametricregime}) or (\ref{eq:parametricregime}) are given in Section~\ref{sec:appsec}. Under the hypothesis (\ref{eq:nonparametricregime})  we may solve for the critical radii $\gamma \asymp \sqrt{\frac{\log \mathcal{N}(\mathbb{F},\rho,\gamma)}{T}} $ and $ \delta \asymp \sqrt{\frac{\log \mathcal{N}(\mathbb{F},\rho,\delta)}{\sigma_T^2(\mathbb{F},P_\star)}}$. This leads to the following result.
\begin{theorem}[Nonparametric Rates]
\label{thm:nonparametricrates}
Under the hypotheses of Theorem~\ref{thm:ezpzthm} and if further inequality (\ref{eq:nonparametricregime}) holds for  $p,q \in \mathbb{R}_+$ with $q<2$, the least squares estimator (\ref{df:PEM}) satisfies
\begin{align}
\label{eq:nonparathm}
\E\|\hat f(\xi)-f_\star(\xi)\|_2 \lesssim \sqrt{\frac{1}{2-q}}\left(\frac{p\sigma^2_w }{T}\right)^{\frac{1}{2+q}}+ \left(p\sigma_T^2(\mathbb{F},P_\star)\right)^{\frac{1}{2+q}}.
\end{align}
\end{theorem}
A similar statement holds for $q\geq 2$ and can be found following the proof of Theorem~\ref{thm:nonparametricrates} in Appendix~\ref{subsec:thmnonpara}. We also have a version of the above theorem, proven in Appendix~\ref{subsec:thmpara}, applicable to the parametric entropy growth regime. 
\begin{theorem}[Parametric Rates]
Under the hypotheses of Theorem~\ref{thm:ezpzthm} and if further inequality (\ref{eq:parametricregime}) holds for  $p,c \in \mathbb{R}_+$, the least squares estimator (\ref{df:PEM}) satisfies
\label{thm:parametricrates}
\begin{align*}
\E \|\hat f(\xi)-f_\star(\xi)\|_2 &\lesssim \sqrt{\frac{\sigma^2_w p \log (1 + c\sqrt{d_y} \sigma_w T^2 )}{T}} +\sqrt{\sigma_T^2(\mathbb{F},P_\star)p\log (1 + c T )}+\frac{1}{T}.
\end{align*}
\end{theorem}

The error estimates given in Theorems~\ref{thm:nonparametricrates} and \ref{thm:parametricrates} are rate-optimal\footnote{Modulo a logarithmic term for the parametric regime.} in terms of $T$ whenever the generalization term satisfies $\sigma_T^2(\mathbb{F},P_\star)\lesssim 1/T$. We  show that this rate of decay of the subgaussian parameter $\sigma_T^2(\mathbb{F},P_\star)$ holds for instance when the autoregressive dynamics (\ref{eq:ards}) are contracting or more generally when the process (\ref{eq:ds}) is mixing. In some sense, $\sigma_T^2(\mathbb{F},P_\star)$ is a measure of the magnitude of the process (\ref{eq:ds}) and its correlation length. We develop this idea next in Section~\ref{sec:gensec}.

\subsection{Sufficient Conditions for Generalization: Stability and Mixing}
\label{sec:gensec}
As noted above, we crucially need conditions for which the spatiotemporal variance proxy satisifes $\sigma_T^2(\mathbb{F},P_\star) \lesssim 1/T$. While this scaling typically holds for i.i.d. data, we show that it also holds for temporally correlated data arising from a single trajectory under suitable stability or mixing assumptions on the dynamics \eqref{eq:ds}. 

\paragraph{Contracting systems} When working with Lipschitz systems of the form (\ref{eq:ards}) one can relate the parameter $\sigma_T^2(\mathbb{F},P_\star)$ to the stability of the map $f_\star$. Fix a norm $\|\cdot\|_{\mathsf{X}}$ on $\mathsf{X}$. We say that $f_\star$ is ($L_\star$,$\|\cdot\|_{\mathsf{X}}$)-contractive if for some $L_\star< 1$, we have that $\|f(x)-f(z)\|_{\mathsf{X}} \leq L_\star \|x-y\|_{\mathsf{X}}$ for all $x,z \in \mathsf{X}$.


\begin{proposition}[Contraction Implies Generalization]
\label{prop:cig}
Suppose that we are in the autoregressive setting (\ref{eq:ards}), that $f_\star$ is ($L_\star$,$\|\cdot\|_{\mathsf{X}}$)-contractive, and that all functions $f\in \mathbb{F}$ are $L$-Lipschitz with respect to $\|\cdot\|_2$. If further $\sup_{x\in \mathsf{X}} \|x\|_2\leq B$, it holds that
\begin{align*}
\sigma_T^2(\mathbb{F},P_\star)\leq  64 \cdot  \frac{ M^2 B^2L^2 }{m^2(1-L_\star)^2 T}
\end{align*}
where
\begin{align*}
M &= \sup_{x\neq x'} \frac{\|x-x'\|_2}{\|x-x'\|_{\mathsf{X}}} &\textnormal{ and}  & &m =\sup_{x\neq x'} \frac{\|x-x'\|_\mathsf{X}}{\|x-x'\|_{2}}.
\end{align*}
\end{proposition}

The proof of Prop.~\ref{prop:cig} combines an Azuma-McDiarmid-Hoeffding like argument with a stability argument, and can be found in Appendix~\ref{sec:stabproofs}. Proposition~\ref{prop:cig} allows us to further simplify the bound (\ref{eq:nonparathm}) whenever $f_\star$ is contractive. Namely, when $\mathbb{F}$ is a bounded subset of $L$-Lipschitz functions with metric entropy scaling as $p \delta^{-q}$ and $q<2$, the bound in Theorem~\ref{thm:nonparametricrates} for the autoregressive system (\ref{eq:ards}) becomes
\begin{align*}
\E\|\hat f(\xi)-f_\star(\xi)\|_2 \lesssim \sqrt{\frac{1}{2-q}}\left(\frac{\sigma^2_w p}{T}\right)^{\frac{1}{2+q}}+ \left( \frac{M^2 B^2L^2 }{m^2(1-L_\star)^2 T}\right)^{\frac{1}{2+q}}.
\end{align*}

This bound shows that more stable systems, as captured by the Lipschitz constant $L_\star$, have smaller generalization error. This interpretation is in line with the recent trend of using stability bounds to study learning algorithms applied to data generated by a dynamical system, see for example \cite{boffi2021regret} and \cite{tu2021sample}. Finally we note that although we restricted our analysis to contracting systems, our results are easily extended to a more general notion of nonlinear stability.  In particular, we extend Proposition~\ref{prop:cig} in Appendix~\ref{sec:ext2ISS} to systems satisfying a notion of exponential incremental input-to-state stability, a standard notion from robust nonlinear control theory \citep{angeli2002lyapunov}.

\paragraph{Mixing systems} Alternatively, one may prefer to work with a stochastic notion of stability: we now demonstrate that our approach is equally applicable to mixing systems. To this end, we recall the definition of a mixing time: if $\{x_t\}_{t\in \mathbb{Z}_{\geq 0}}$ is a Markov chain with transition kernel $P(x,\cdot)$ and invariant measure $\nu_\infty$, its mixing time $t_{\mathsf{mix}}$ is given by $t_{\mathsf{mix}}\triangleq\min\{t \in \mathbb{N}: \sup_{x\in X} d_{TV}(P^t(x,\cdot), \nu_\infty)\leq 1/4\}$. Equipped with this notion, the following result can be inferred from \cite{paulin2015concentration} (see Definition 1.3 and Corollary 2.10 therein).

\begin{proposition}[Mixing Implies Generalization]
\label{prop:mig}
Assume  that $\sup_{y\in \mathsf{Y}} \|y \|_2\leq B$. If the sequence $\{x_t\}_{t\in \mathbb{Z}_{\geq 0}}$ in system (\ref{eq:ds}) is a $t_{\mathsf{mix}}$-mixing Markov chain, then the class $\mathbb{F}$ is $\sigma_T^2(\mathbb{F},P_\star)$-subgaussian with
\begin{align*}
\sigma_T^2(\mathbb{F},P_\star) \lesssim \frac{B^2t_{\mathsf{mix}}}{T}.
\end{align*}
\end{proposition}

Comparing Proposition~\ref{prop:mig} with Proposition~\ref{prop:cig} we see that we obtain similar bounds on the subgaussian parameter $\sigma_T^2(\mathbb{F},P_\star)$. While Proposition~\ref{prop:cig} is easier to prove and has a direct interpretation in terms of the model (\ref{eq:ards}), Proposition~\ref{prop:mig} has the advantage of being equally applicable to both the more general time series model (\ref{eq:ds}) and the dynamical system (\ref{eq:ards}).

\subsection{Summary}
Our results show that a large class of nonlinear systems can be learned at the minimax optimal nonparametric rate $T^{-1/(2+q)}$ using the LSE \eqref{df:PEM}. This significantly extends our current understanding of nonasymptotic learning of dynamical systems from single trajectory data. By contrast, previous work assumes i.i.d. data or focuses on either linear models \citep{simchowitz2018learning, tsiamis2019finite, jedra2020finite} or parametric models with known nonlinearities \citep{foster2020learning, sattar2020non, mania2020active, jain2021near}. 


\section{Proof Strategy for Theorem~\ref{thm:ezpzthm}}
\label{sec:proofoutline}

Our analysis of the least squares estimator (\ref{df:PEM}) begins with the following information-theoretic decoupling estimate inspired by \cite{russo2019much} and \cite{NIPS2017_ad71c82b}.

\begin{proposition}
\label{genpredprop}
 Let $f$ and $g$ be random variables (functions) taking values in $\mathbb{F}$. If $\mathbb{F}$ is $\sigma_T^2(\mathbb{F},P_\star)$-subgaussian with respect to dynamics \eqref{eq:ds}, we have that
\begin{align}
\label{eq:decoupling}
\E  \| f(\xi)- g(\xi) \|_2  \leq \sqrt{\frac{1}{T}\E \sum_{t=0}^{T-1} \| f(x_t)-g(x_t) \|^2_2 }+\sqrt{2\sigma_T^2(\mathbb{F},P_\star) I((f,g); Z)},
\end{align}
where $Z=(x_0,\dots,x_{T-1},y_0,\dots,y_{T-1})$ and where $\xi$ has uniform mixture distribution over the covariates $(x_0,\dots,x_{T-1})$ of system (\ref{eq:ds}) and is independent of all other randomness.
\end{proposition}
The proof of the estimate (\ref{eq:decoupling}) relies on the Donsker-Varadhan variational representation of relative entropy and is given in Appendix~\ref{sec:proofofgenpred}.

To arrive at Theorem~\ref{thm:ezpzthm} we set $f=\hat f_\delta$, the discretized least squares estimator (\ref{df:PEMdisc}), and $g=f_\star$ in inequality (\ref{eq:decoupling}). Now, the discretized estimator $\hat f_\delta$ behaves similarly to $\hat f$ since the covering $\mathbb{F}_\delta$ to which $\hat f_\delta$ belongs is with respect to the uniform metric $\|\cdot\|_\infty$; that is $\|\hat f_\delta - \hat f\|_\infty \leq \delta$. Exploiting this similarity in behavior between $\hat f$ and $\hat f_\delta$ yields a bound of the form
\begin{align}
\label{eq:decouplingused}
    \E  \| \hat f(\xi)- f_\star(\xi) \|_2  \lesssim\sqrt{\frac{1}{T}\E \sum_{0=1}^{T-1} \| \hat f(x_t)-f_\star(x_t) \|^2_2 }+\delta + \sqrt{2\sigma_T^2(\mathbb{F},P_\star) I(\hat f_\delta; Z)}.
\end{align}
The advantage of inequality (\ref{eq:decouplingused}) over directly choosing $f=\hat f$ in inequality (\ref{eq:decoupling}) is that the mutual information term in (\ref{eq:decouplingused}) is with respect to $\hat f_\delta$ instead of $\hat f$. By finiteness of $\mathbb{F}_\delta$ this mutual information term is readily controlled by the metric entropy: $I(\hat f_\delta; Z)\leq \log \mathcal{N}(\mathbb{F},\|\cdot\|_\infty,\delta)$. This yields the second term appearing in inequality (\ref{eq:metentthm}) of Theorem~\ref{thm:ezpzthm} which controls the generalization performance of the discretized estimator $\hat f_\delta$. It remains to control the first term appearing on right of inequality (\ref{eq:decouplingused}).
\subsection{Offset Basic Inequality Analysis}
\label{sec:offsetsec}

We now describe our analysis of the in-sample prediction (or training) error, namely the first term appearing on the right hand side of inequality (\ref{eq:decouplingused}). We start with an inequality due to  \cite{liang2015learning}, which is a variant of the basic inequality of least squares and that is crucial to analyzing the in-sample prediction error for correlated data.
\begin{lemma}
\label{lemma:basicineq}
For the system (\ref{eq:ds}) the LSE (\ref{df:PEM}) satisfies 
\begin{align}
\label{eq:basicineq}
\frac{1}{T} \sum_{t=0}^{T-1} \| \hat f(x_t) - f_\star(x_t)\|_2^2  \leq \frac{1}{T} \sup_{f\in \mathbb{F}_\star} \sum_{t=0}^{T-1}4\langle w_t, f(x_t)\rangle-\| f(x_t)\|^2_2.
\end{align}
\end{lemma}

Inequality (\ref{eq:basicineq}) implies that it suffices to control the family of tilted random walks with increments $4\langle w_t, f(x_t)\rangle-\|f(x_t)\|_2^2$. The right hand side of equation (\ref{eq:basicineq}) is the supremum of a stochastic process over $\mathbb{F}_\star$. This becomes more clear if we define
\begin{align}
\label{eq:mtcontdef}
M_T(f) \triangleq \sum_{t=0}^{T-1}4 \langle w_t, f(x_t)\rangle-\|f(x_t)\|_2^2,
\end{align}
which for each fixed $T$ is a real-valued process over $\mathbb{F}_\star$. The supremum of the process $M_T(f)$ in (\ref{eq:mtcontdef}) can be viewed as a self-normalized version of the (subgaussian) complexity of $\mathbb{F}_\star$; as the next lemma shows, regardless of the member $f$ and the time-horizon $T$, its scale is always unity. 
\begin{lemma}
\label{lemma:MTscale}
For any function space $\mathbb{F}$,  any $f\in \mathbb{F}_\star$ and $\lambda \in [0,1/2\sigma_w^2]$ we have that
\begin{align*}
\E \exp\left( \lambda M_T(f) \right) \leq 1.
\end{align*}
\end{lemma}
This leads to the maximal inequality (\ref{eq:maxineqselfnorm}) of Lemma~\ref{lemma:selfnormmetent} below.

\begin{lemma}
\label{lemma:selfnormmetent}
Let $S$ be a finite subset of the shifted metric space $\mathbb{F}_\star$. Then
\begin{align}
\label{eq:maxineqselfnorm}
\E \sup_{f\in S} M_T(f) \leq 2\sigma_w^2\log |S|. 
\end{align}
\end{lemma}
As observed by \cite{liang2015learning}, if we had not included the offset term $-\|f(x_t)\|^2$, a naive bound would have yielded $\E \sup_{f\in S} M_T(f) \lesssim \sqrt{ T   \ln |S|}$, penalizing us by a factor $\sqrt{T}$ for the scale of $\sum_{t=0}^{T-1} \langle w_t, f(x_t) \rangle$. 

If we are given a finite class $|\mathbb{F}|<\infty$, we can take $\delta=0$ in (\ref{eq:decouplingused}) and Lemma~\ref{lemma:selfnormmetent} in combination with Proposition~\ref{genpredprop} directly give generalization bounds with optimal dependency on $\log |\mathbb{F}|$, see Appendix~\ref{sec:finiteclasses} for details. For large spaces with metric structure, we combine this analysis with discretization  and chaining to arrive at Theorem~\ref{thm:ezpzthm}, see Appendix~\ref{sec:mainproofsec}.

\section{Applications of Theorem~\ref{thm:ezpzthm}}
\label{sec:appsec}

Having outlined the proof of Theorem~\ref{thm:ezpzthm}, this section is devoted to various examples, demonstrating that we obtain tight rates. As a first example,  consider the learning Lipschitz dynamics on $\mathbb{R}$. Combining Theorem~\ref{thm:nonparametricrates} with Proposition~\ref{prop:cig}, and  that $1$-dimensional $1$-Lipschitz functions  exhibit entropy growth $\log \mathcal{N}(\mathbb{F},\rho,\delta) \lesssim \frac{1}{\delta}$ the following is immediate.

\begin{example}
\label{ex:lip}
Suppose that we are in the autoregressive setting (\ref{eq:ards}). If $$\mathbb{F} = \{f:  \mathbb{R} \to [0,1] \:;\: \textnormal{$f(0)=0$ and each $f$ is $1$-Lipschitz}\} $$ with $f_\star\in \mathbb{F}$ being $L_\star$-Lipschitz, $L_\star <1$, then the least squares estimator (\ref{df:PEM}) achieves
\begin{align*}
\E \|\hat f(\xi) - f_\star(\xi)\|_2 \lesssim  \left[\frac{1}{(1-L_\star)^{2/3}}+\sigma_w^{2/3}\right] T^{-1/3}.
\end{align*}
The factor $T^{-1/3}$ is optimal, and appears even in the i.i.d. case, see for example \cite[Ch. 13]{wainwright2019high}.
\end{example}

More generally, growth rates of the form $\log \mathcal{N}(\mathbb{F}_\star,\rho,\delta) \lesssim p \left(\frac{1}{\delta}\right)^{q}$, as considered in Theorem~\ref{thm:nonparametricrates}, are typical for spaces comprised of smooth functions as computed in \cite[Sec. 5]{kolmogorov1961entropy}. For instance, it can be shown that the space $C^\beta(\mathsf{X} \to \mathsf{Y})$ of $\beta$-times differentiable functions between connected compact subsets of Euclidean space $\mathsf{X}$ and $\mathsf{Y}$ can be bounded as\footnote{This is a consequence of Theorem XIII in \cite{kolmogorov1961entropy}.}
\begin{align*}
\log \mathcal{N}(C^\beta(\mathsf{X} \to \mathsf{Y}),\|\cdot\|_\infty,\delta) \lesssim d_y^{1+\frac{d_x}{2\beta}} \left(\frac{1}{\delta}\right)^{\frac{d_x}{\beta}}.
\end{align*}

In the following example, we consider function spaces $\mathbb{F}$ specified by a Reproducing Kernel Hilbert Space (RKHS) satisfying an eigenvalue decay condition.

\begin{example}
\label{ex:rkhs}
Consider model (\ref{eq:ds}) with $d_y =1$, and suppose that $f_\star$ is known to belong to some RKHS $\mathbb{H}$. More precisely, we fix a probability measure $\mathbf{P}$ on a compact set $\mathsf{X}$ and let $K :\mathsf{X} \times \mathsf{X}\to \mathbb{R}$ be a differentiable positive semidefinite kernel function. Assume $K$ has eigenexpansion $K(x,z) = \sum_{i=1}^\infty\lambda_i \phi_i(x) \phi_i(z)$ where $\{\phi_i\}_{i=1}^\infty$ is an  orthonormal basis of $L^2(\mathbf{P})$, and where $\{\lambda_i\}_{i=1}^\infty$ is a sequence of nonnegative real numbers. Recall that the RKHS associated to $K$ is then given by
\begin{align*}
\mathbb{H} = \left\{ f = \sum_{i=1}^\infty b_i \phi_i \, \Big| \, \{b_i\} \subset l^2(\mathbb{N}), \sum_{i=1}^\infty \frac{b_i^2}{\lambda_i} < \infty \right\}.
\end{align*}
Denote also by $\mathbb{B}_\mathbb{H}$ the unit ball in $\mathbb{H}$ induced by the inner product  $\displaystyle\langle f,g \rangle_{\mathbb{H}}  = \sum_{i=1}^\infty \frac{\langle f,\phi_i \rangle \langle g,\phi_i\rangle }{\lambda_i}$, where $\langle \cdot , \cdot \rangle$ is the standard inner product in $L^2(\mathbf{P})$. See Chapter 12 of \cite{wainwright2019high} for further background.

 Suppose the kernel $K$ satisfies the regularity conditions $\sup_{x\in \mathsf{X}} | \phi_j (x)| \leq A$ and  $\lambda_j\lesssim j^{-2\alpha}$ for some $\alpha>1/2$ and all $j\in \mathbb{N}$. Assume further that $\sup_{y\in\mathsf{Y}} |y| \leq B$ and that $(x_0,\dots,x_{T_1})$ is a Markov chain with finite mixing time $\tau$. Then if $f_\star \in \mathbb{B}_\mathbb{H}$, the least squares estimator $\hat f$ defined in equation (\ref{df:PEM}) with $\mathbb{F}=\mathbb{B}_\mathbb{H}$ satisfies
\begin{multline*}
    \E\|\hat f(\xi) - f_\star(\xi)\|_2 \\
    \lesssim  \left(A^{1/(2\alpha+1)} \sqrt{\frac{\alpha}{2\alpha-1}}\left(\frac{\sigma^2_w }{T}\right)^{\frac{\alpha}{2\alpha+1}}+ A^{1/(2\alpha+1)}\left(\frac{B^2t_{\mathsf{mix}}}{T}\right)^{\frac{\alpha}{2\alpha+1}}\right)\polylog \left(1+ \lambda_1 A T \right).
\end{multline*}
The result relies on a metric entropy calculation of $\mathbb{B}_{\mathbb{H}}$ in the supremum metric, which can be found in Appendix~\ref{sec:metentcalcrkhs}. Notice that a faster eigenvalue decay corresponds to a faster rate of convergence. In particular, as $\alpha \to \infty$ the rate of convergence approaches the parametric rate $T^{-1/2}$, modulo a polylogarithmic factor. Again, the rate $T^{-\alpha/(2\alpha+1)}$ is optimal even in the i.i.d. setting ($t_{\mathsf{mix}}=1$). A supporting experiment can be found in Appendix~\ref{sec:experiments}.
\end{example}

The next example revisits the generalized linear models using Theorem~\ref{thm:parametricrates}. This model has recently been analyzed using recursive methods \citep{foster2020learning, sattar2020non, jain2021near}. 

\begin{example}
\label{ex:gls}
Consider a system of the form
\begin{align}
\label{gls}
x_{t+1} = \phi(A_\star x_t) +w_t
\end{align}
and define
\begin{align*}
\mathbb{F}^\phi = \{f\in C(X\to X); f(\cdot) = \phi(A \: \cdot \:) \textnormal{ with } A \in \mathbb{R}^{d_x\times d_x} \textnormal{ and } \| A\|_F \leq C \}
\end{align*} 
for some $C>0$ and where the Frobenius norm of $A$ is given by $\|A\|_F= \sqrt{\tr A^\top A}$. This setting is a special case of the autoregressive system (\ref{eq:ards}) with $ f_\star(\cdot) = \phi(A_\star \: \cdot \:)$. 

If $\phi$ is $1$-Lipschitz with respect to $\|\cdot\|_2$,  $\sigma_{\max}(A_\star) \leq L_\star$ so that $\phi(A_\star \cdot)$ is $L_\star$-contractive.  Suppose further $\sup_{x\in \mathsf{X}}\|x\|_2\leq B$, then  the least squares estimator $\hat f$  (\ref{df:PEM}) for the system (\ref{gls}) with hypothesis class $\mathbb{F}=\mathbb{F}^\phi$ satisfies the error bound 
\begin{align*}
\E\| \hat f(\xi) - \phi(A_\star \xi)\|_2  \lesssim \sqrt{\frac{\sigma_w^2}{T}  d_x^2 \log (1+2\sigma_w \sqrt{d_x} BCT^2)}+\sqrt{\frac{B^2}{(1-L_\star)^2 T} d_x^2 \log (1+2BCT)}.
\end{align*}
A proof of this claim can be found in Appendix~\ref{sec:ex:gls}. The dependency on $T, d_x$ and $L_\star$ matches Theorem 2 of
 \cite{foster2020learning},  which is optimal in $T$ and $d_x$. 
\end{example}

\section{Discussion}
We have leveraged recently developed information-theoretic tools \citep{russo2019much, NIPS2017_ad71c82b} to analyze the nonparametric LSE (\ref{df:PEM}) for learning dynamical systems. Our analysis yields, to the best of our knowledge, the first rate-optimal bounds for nonparametric estimation of stable or otherwise mixing nonlinear systems from a single trajectory.  In addition, our results are able to capture, as a special case, existing parametric rates in the literature \citep{foster2020learning,sattar2020non}.

While our bounds are in expectation, similar tools applied via exponential stochastic inequalities have recently been used to provide high probability generalization bounds for statistical learning  \citep{hellstrom2020generalization, grunwald2021pac}. Combining our results with these methods could potentially also yield control of  $\|\hat f(\xi)-f_\star(\xi)\|_2$ with high probability, and is an exciting direction for future work. To arrive at our bounds, we leveraged the decoupling technique of \cite{russo2019much} and \cite{NIPS2017_ad71c82b}. To apply these techniques to the system (\ref{eq:ds}), we had to control the subgaussian parameter $\sigma_T^2(\mathbb{F},P_\star)$, which captures the spatiotemporal spread of the covariates. We showed that this term can be controlled using either control-theoretic stability notions, or more general mixing properties.

While this paper develops tools to estimate $f_\star$ in system (\ref{eq:ds}), we believe that the general technique developed is more broadly applicable, and of independent interest. For example, since the variance proxy $\sigma_T^2(\mathbb{F},P_\star)$ captures the subgaussian parameter of the loss function in statistical learning \citep{NIPS2017_ad71c82b}, we prove in Appendix~\ref{subsec:genbounds} a generalization bound for single-trajectory learning for Lipschitz loss functions. 

Finally, an open problem is to determine for which learning problems the system (\ref{eq:ds}) is required to mix in the single trajectory setting. Most previous works on learning in nonlinear dynamical systems rely on similar mixing time or stability arguments. The cost of this is typically a multiplicative factor in the final bound that degrades as stability is lost  \citep{foster2020learning, sattar2020non, boffi2021regret}. In contrast, it is well-known that this dependency can be avoided for learning in linear systems \citep{lai1982least, simchowitz2018learning}. Recently \cite{jain2021near} showed under a strong invertibility condition that dependency on the mixing time can also be avoided  for the generalized linear model (\ref{gls}). This leaves open the question whether learning without mixing is possible in situations beyond the generalized linear model.
\paragraph{Acknowledgements} 
Ingvar Ziemann and Henrik Sandberg are supported by the Swedish Research Council (grant 2016-00861). Nikolai Matni is supported in part by NSF awards CPS-2038873 and CAREER award ECCS-2045834, and a Google Research Scholar award.

\appendix
\newpage
\tableofcontents
\newpage
\section{Proof of Theorem~\ref{thm:ezpzthm} and its Corollaries}
\label{sec:mainproofsec}
We now turn to the proof of Theorem~\ref{thm:ezpzthm}. First, we begin by applying Proposition~\ref{genpredprop} to the discretized estimator $\hat f$. Let us begin by bounding the generalization error of the quantized estimator $\hat f_\delta$ as defined by (\ref{df:PEMdisc}). We may write
\begin{equation}
\begin{aligned}
\label{eq:fast1}
\E \|\hat f_\delta(\xi)-f_\star(\xi)\|_2 &\leq\sqrt{\int \frac{1}{T}\sum_{t=1}^T \|\hat f_\delta(x_t)-f_\star(x_t) \|^2_2  dP_{x^T,\hat f_\delta}}+\sqrt{2\sigma_T^2(\mathbb{F},P_\star) I(\hat f_\delta; x^T)}\\
&\leq\sqrt{\int \frac{1}{T}\sum_{t=1}^T \|\hat f_\delta(x_t)-f_\star(x_t) \|^2_2  dP_{x^T,\hat f_\delta}}+\sqrt{2\sigma_T^2(\mathbb{F},P_\star)\log \mathcal{N}(\mathbb{F},\rho,\delta)} .
\end{aligned}
\end{equation}
The first  inequality is just Proposition~\ref{genpredprop} applied to $\hat f_\delta$, whereas the second inequality follows from the fact that $ I(\hat f_\delta; x^T) \leq \log \mathcal{N}(\mathbb{F},\rho,\delta)$, since the random variable $\hat f_\delta$ can take on at most $\mathcal{N}(\mathbb{F},\rho,\delta)$ different values and the bound $I(\hat f_\delta;x^T) \leq H(\hat f_\delta) \leq\log \mathcal{N}(\mathbb{F},\rho,\delta)$. 

It remains to bound the in-sample-prediction error. We have
\begin{align}
\label{eq:fast2}
\int \frac{1}{T}\sum_{t=1}^T \|\hat f_\delta(x_t)-f_\star(x_t) \|^2_2  dP_{x^T,\hat f_\delta} \leq \int \frac{2}{T}\sum_{t=1}^T \|\hat f(x_t)-f_\star(x_t) \|^2_2  dP_{x^T,\hat f} +2\delta^2
\end{align}
by construction of $\hat f_\delta$, the parallelogram law, and the fact that $\mathbb{F}$ is metrized by the supremum norm.

The main technincal chaining step is given in Lemma~\ref{lemma:chainlemma}. Namely, by  appealing to Lemma~\ref{lemma:chainlemma}  and combining with (\ref{eq:fast1}) and (\ref{eq:fast2}) we find

\begin{multline}
\label{eq:slow3}
\E \|\hat f(\xi)-f_\star(\xi)\|_2 \\
\leq\sqrt{\frac{8\sigma^2_w\log \mathcal{N}(\mathbb{F},\rho,\gamma)}{T}+128\alpha\sigma_w\sqrt{d_y}+64\frac{\sigma_w } {\sqrt{T}} \int_\alpha^\gamma \sqrt{\log \mathcal{N}(\mathbb{F},\rho,s)}ds  +2\delta^2}\\
+\delta+\sqrt{2\sigma_T^2(\mathbb{F},P_\star)\log \mathcal{N}(\mathbb{F},\rho,\delta)}.
\end{multline}
since $\E \|\hat f(\xi)-f_\star(\xi)\|_2\leq\E \|\hat f_\delta(\xi)-f_\star(\xi)\|_2+\delta$ by the triangle inequality. The result follows after  pulling the factor $2\delta^2$ out of the square root sign using the triangle inequality.
\hfill $\blacksquare$

For large spaces and fine grained coverings, the metric entropy starts to dominate the scale free process $M_T(f)$ appearing in Lemma~\ref{lemma:selfnormmetent}.  The analysis of $M_T(f)$ in Lemma~\ref{lemma:chainlemma} below essentially follows that in \cite{liang2015learning} (compare with their Lemma 6) with certain slight simplifications due to the added structure the uniform topology on $C(\mathsf{X} \to \mathsf{Y})$ affords us. We begin with an analogue of Lemma~\ref{lemma:selfnormmetent} which takes the scale of the functions considered into account.

The proof of Theorem~\ref{thm:ezpzthm} requires both  Lemma~\ref{lemma:rootmetent} and Lemma~\ref{lemma:selfnormmetent} to accomplish succesful chaining for a variety of metric entropy scalings. While these results are quite similar, for typically metric entropy scalings ($q<2$) Lemma~\ref{lemma:selfnormmetent} performs better due to self-normalization. However, once $q\geq 2$ Dudley's entropy integral becomes singular near $0$, wherefore we also require the cruder discretization provided by Lemma~\ref{lemma:rootmetent}. In other words, there is a critical radius $r \asymp \sqrt{\frac{\log \mathcal{N}(\mathbb{F}_\star,\rho,r)}{T}}$ below which self-normalization becomes insignificant and the supremum norm bound $\|f\|_\infty\leq r$ starts to become increasingly important to control the supremum of $M_T(f)$.

\begin{lemma}
\label{lemma:chainlemma}
Fix a metric space $(\mathbb{F},\rho)$ with $\rho(f,g) = \sup_{x\in X} \|f(x) -g(x)\|_2$. Then with $M_T(f)$ defined by equation (\ref{eq:mtcontdef}), we have that
\begin{align*}
&\E \sup_{f\in \mathbb{F}_\star} \frac{1}{T} M_T(f)\\ &\leq \inf_{\gamma> 0, \alpha \in [0,\gamma]} \Bigg\{ \frac{8\sigma^2_w\log \mathcal{N}(\mathbb{F},\rho,\gamma)}{T}+128\alpha\sigma_w \sqrt{d_y}+64\frac{\sigma_w } {\sqrt{T}} \int_\alpha^\gamma \sqrt{\log \mathcal{N}(\mathbb{F},\rho,s)}ds  \Bigg\}.
\end{align*}
\end{lemma}

 As noted in \cite{liang2015learning}, the optimal value for $\gamma$ in Lemma~\ref{lemma:chainlemma}  is of the same nature as when obtained by other methods, see for example Chapter 13 of \cite{wainwright2019high} for a more standard approach. 
 
 The idea below is to decompose the supremum over $\mathbb{F}_\star$ in inequality (\ref{eq:basicineq}) by $\mathbb{F}_\star = \{\mathbb{F}_\star\cap \alpha \mathbb{B}_\star \} \cup \{\mathbb{F}_\star\setminus \alpha \mathbb{B}_\star \}$ where $\mathbb{B}_\star$ is the unit ball in the space of bounded functions, centered at $f_\star$. On $\mathbb{F}_\star\cap \alpha \mathbb{B}_\star $, $M_T(f)$, the process (\ref{eq:mtcontdef}), is small since $\|f\|_\infty \leq  \alpha$. The role of chaining is to show that it suffices to approximate $M_T(f)$ on $\mathbb{F}_\star\setminus \alpha \mathbb{B}_\star$ at low resolution and thus rely on Lemma~\ref{lemma:selfnormmetent} with small $|S|$.

\begin{proof}
Observe first that for any fixed $\alpha>0$ we have, simply by discarding the negative second order term, Cauchy-Schwarz, and a standard subgaussian concentration inequality for $\E \|w_t\|_2$:
\begin{align}
\label{eq:smallestscale}
\E \sup_{f\in \mathbb{F}_\star \cap \alpha \mathbb{B}_\star} \frac{1}{T} M_T(f)  \leq \E \sup_{f\in \mathbb{F}_\star \cap \alpha \mathbb{B}} \frac{4}{T} \sum_{t=0}^{T-1} \langle w_t, f(x_t)\rangle  \leq 16  \alpha \sigma_w \sqrt{d_y}.
\end{align}
A standard one-step discretization bound (c.f. the proof of Proposition 5.17 in \cite{wainwright2019high}) combined with the finite class maximal inequality of Lemma~\ref{lemma:selfnormmetent} yields for fixed $\gamma  > 0$:
\begin{align*}
\E \sup_{f\in \mathbb{F}_\star} \frac{1}{T} M_T(f) \leq   \frac{8\sigma^2_w\log \mathcal{N}(\mathbb{F}_\star,\rho,\gamma)}{T}+ 2\E  \sup_{f\in \mathbb{F}_\star\cap \gamma \mathbb{B}_\star} \frac{1}{T} M_T(f)
\end{align*}
Having extracted the fast rate term for scales larger than $\gamma$, we proceed with a chaining bound on the second term above. Since $M_T(f)$ satisfies the maximal inequality (\ref{eq:naivesubgMT}) with $r=\gamma$, chaining (as in Theorem 5.22 of \cite{wainwright2019high}) yields
\begin{align*}
2\E  \sup_{f\in \mathbb{F}_\star\cap \gamma \mathbb{B}_\star} \frac{1}{T} M_T(f) \leq  32  \alpha \sigma_w \sqrt{d_y}+ 64\frac{\sigma_w}{\sqrt{T}} \int_{\alpha/4}^\gamma \sqrt{\log \mathcal{N}(\mathbb{F}_\star,\rho,s)}ds.
\end{align*}
Note that $ \mathcal{N}(\mathbb{F}_\star,\rho,s) = \mathcal{N}(\mathbb{F},\rho,s)$ by translation invariance of the metric $\rho$. The results now follows by terminating the chaining at scale $\alpha$, using (\ref{eq:smallestscale}) to bound that which remains and rescaling $\alpha \leftrightarrow 4\alpha$.
\end{proof}

\subsection{Proof of Theorem~\ref{thm:nonparametricrates}}
\label{subsec:thmnonpara}

Under the hypothesis (\ref{eq:nonparametricregime}) we may use Theorem~\ref{thm:ezpzthm} with $\alpha=0$ to write
\begin{equation}
\begin{aligned}
\label{eq:nonpara1}
&\E \|\hat f(\xi) - f_\star(\xi)\|_2 \\
&\lesssim \sqrt{\frac{8\sigma^2_wp \left(\frac{1}{\gamma}\right)^{q}}{T}+64\frac{\sigma_w } {\sqrt{T}} \int_0^\gamma \sqrt{p \left(\frac{1}{s}\right)^{q}}ds  }+3\delta +\sqrt{2\sigma_T^2(\mathbb{F},P_\star)p \left(\frac{1}{\delta}\right)^{q}}\\
&=\sqrt{\frac{8\sigma^2_wp \left(\frac{1}{\gamma}\right)^{q}}{T}+64\frac{\sigma_w } {\sqrt{T}} \sqrt{p}\frac{2}{2-q} \left(\frac{1}{\gamma} \right)^{1-q/2}  }+3\delta +\sqrt{2\sigma_T^2(\mathbb{F},P_\star)p \left(\frac{1}{\delta}\right)^{q}}.
\end{aligned}
\end{equation}
Choosing $\gamma$ and $\delta$ to satisfy the optimal balance: $\gamma \asymp \left(\frac{\sigma^2_w p }{T}\right)^{\frac{1}{2+q}}$ and $\delta \asymp \left( p\sigma_T^2(\mathbb{F},P_\star)\right)^{\frac{1}{2+q}}$ (\ref{eq:nonpara1}) becomes
\begin{equation*}
\begin{aligned}
&\E \|\hat f(\xi) - f_\star(\xi)\|_2  \\
&\lesssim\sqrt{\frac{2}{2-q}\frac{\sigma^2_wp  \left(\frac{\sigma^2_w p }{T}\right)^{\frac{-q}{2+q}}}{T}}+ \left(p\sigma_T^2(\mathbb{F},P_\star)\right)^{\frac{1}{2+q}}\\
&\asymp \sqrt{\frac{2}{2-q}}\left(\frac{\sigma_w^2 p}{T}\right)^{\frac{1}{2+q}}+ \left(p\sigma_T^2(\mathbb{F},P_\star)\right)^{\frac{1}{2+q}}.
\end{aligned}
\end{equation*}
This verifies the claim. \hfill $\blacksquare$

\paragraph{Remark:}
If instead $q> 2$, we have
\begin{align*}
\E\|\hat f(\xi)-f_\star(\xi)\|_2 \lesssim\sigma_w \sqrt{d_y}\left(\frac{1}{q-2} \right)^{1/q}\left(\frac{ p}{d_y T}\right)^{\frac{1}{2q}}+ \left(p\sigma_T^2(\mathbb{F},P_\star)\right)^{\frac{1}{2+q}}
\end{align*}
and similarly but with an extra logarithmic factor at $q=2$. To show this, we again use Theorem~\ref{thm:ezpzthm} but with $\gamma$ chosen sufficiently large such that $\log \mathcal{N}(\mathbb{F}_\star,\rho,\gamma)\asymp 1$. In this case 
\begin{equation*}
\begin{aligned}
&\E \|\hat f(\xi) - f_\star(\xi)\|_2 \\
 &\lesssim \sqrt{\alpha\sigma_w\sqrt{d_y}+\frac{\sigma_w \sqrt{p} } {\sqrt{T}} \left(\frac{2}{q-2} \right)\alpha ^{1-q/2} }+\delta +\sqrt{2\sigma_T^2(\mathbb{F},P_\star)\log \mathcal{N}(\mathbb{F},\rho,\delta)}.
\end{aligned}
\end{equation*}
The claim follows by solving for the optimal balance \begin{align*}
\alpha \asymp  \left(\frac{2}{q-2} \right)^{1/q}\left(\frac{ p}{d_y T}\right)^{\frac{1}{2q}}
\end{align*}
 and $\delta \asymp \left(p\sigma_T^2(\mathbb{F},P_\star)\right)^{\frac{1}{2+q}}$. 

\subsection{Proof of Theorem~\ref{thm:parametricrates}}
\label{subsec:thmpara}

By virtue of Theorem~\ref{thm:ezpzthm} and by selecting $\gamma=\alpha$ we have
\begin{equation}
\begin{aligned}
\label{eq:pararateest}
\E \|\hat f(\xi)-f_\star(\xi)\|_2 &\leq\sqrt{\frac{8\sigma^2_w\log \mathcal{N}(\mathbb{F}_\star,\rho,\alpha)}{T}+128\alpha\sigma_w\sqrt{d_y}  }\\&+3\delta +\sqrt{2\sigma_T^2(\mathbb{F},P_\star)\log \mathcal{N}(\mathbb{F},\rho,\delta)}.
\end{aligned}
\end{equation}
Let now $\alpha = 1/\sqrt{d_y} \sigma_w T^2$ and $\delta = 1/T$. Then (\ref{eq:pararateest}) by using the hypothesis (\ref{eq:parametricregime}) becomes
\begin{equation*}
\begin{aligned}
\E \|\hat f(\xi)-f_\star(\xi)\|_2 &\lesssim \sqrt{\frac{\sigma^2_w p \log (1 + c\sqrt{d_y} \sigma_w T^2 )}{T}+\frac{1}{T^2}}\\&+\frac{1}{T} +\sqrt{2\sigma_T^2(\mathbb{F},P_\star)p\log (1 + c T )}.
\end{aligned}
\end{equation*}
The result follows. \hfill $\blacksquare$

\subsection{Proof of Auxilliary Results}

\paragraph{Proof of Lemma~\ref{lemma:basicineq}}
By optimality of $\hat f$ to the prediction error objective we have that
\begin{align*}
 \sum_{t=0}^{T-1} \| \hat f(x_t) - y_t\|_2^2 \leq \sum_{t=0}^{T-1} \| f_\star(x_t)-y_t\|_2^2.
\end{align*}
Rearranging and expanding the square gives the basic inequality
\begin{align*}
\frac{1}{T} \sum_{t=0}^{T-1} \| \hat f(x_t) - f_\star(x_t)\|_2^2  \leq\frac{2}{T}\sum_{t=0}^{T-1}\langle w_t, \hat f(x_t)-f_\star(x_t)\rangle
\end{align*}
which after multiplying both sides by $2$ can be rearranged again to give 
\begin{align*}
\frac{1}{T} \sum_{t=0}^{T-1} \| \hat f(x_t) - f_\star(x_t)\|_2^2  \leq\frac{1}{T}\sum_{t=0}^{T-1}4\langle w_t, \hat f(x_t)-f_\star(x_t)\rangle-\| \hat f(x_t) - f_\star(x_t)\|_2^2
\end{align*}
so that the result follows by taking the supremum over the variable $\hat f-f_\star\in \mathbb{F}_\star$. \hfill $\blacksquare$

\paragraph{Proof of Lemma~\ref{lemma:selfnormmetent}}
The proof is a straight-forward modification of the standard proof for bounding the expected supremum of subgaussian maxima. By Jensen's inequality and monotonicity of the exponential it follows that
\begin{align*}
\exp  \left(   \lambda \E \sup_{f\in S} M_T(f) \right) & \leq \E\exp  \left(   \lambda  \max_{f\in S} M_T(f) \right)\\
&\leq \E  \max_{f\in S} \exp  \left(   \lambda  M_T(f) \right)\\
&\leq \sum_{f\in S}  \E  \exp  \left(   \lambda  M_T(f) \right).
\end{align*}
Choosing $\lambda = 1/2\sigma_w^2$, application of Lemma~\ref{lemma:MTscale} yields $\exp  \left(   \frac{1}{2\sigma^2} \E \sup_{f\in S} M_T(f) \right) \leq |S|$ which is equivalent to the result. \hfill $\blacksquare$

\paragraph{Proof of Lemma~\ref{lemma:MTscale}}
Write by the tower property
\begin{align*}
&\E\exp \left(\lambda \sum_{t=1}^{T-1}4\langle w_t, f(x_t)\rangle -\|f(x_t)\|_2^2\right)\\
&=\E\exp \left(\lambda \sum_{t=1}^{T-2}4\langle w_t, f(x_t)\rangle -\|f(x_t)\|_2^2\right) \exp\left(-\lambda\|f(x_{T-1})\|_2^2\right)\E_{T-2}\exp \left(\lambda 4\langle w_{T-1}, f(x_{T-1})\rangle \right)\\
&\leq \E\exp \left(\lambda \sum_{t=1}^{T-2}4\langle w_t, f(x_t)\rangle -\|f(x_t)\|_2^2\right) \exp\left(\left[2\lambda^2\sigma_w^2-\lambda\right]\|f(x_{T-1})\|_2^2\right)\\
&\leq  \E\exp \left(\lambda \sum_{t=1}^{T-2}4\langle w_t, f(x_t)\rangle -\|f(x_t)\|_2^2\right)\\
&\leq \dots \leq 1
\end{align*}
as per requirement. \hfill $\blacksquare$

\begin{lemma}
\label{lemma:rootmetent}
Let $S$ be a finite subset of the shifted metric space $\mathbb{F}_\star$ with $\|f\|_\infty \leq r$ for all $f\in S$. Then
\begin{align}
\label{eq:naivesubgMT}
\E \sup_{f\in S} M_T(f) \leq \sqrt{2T\sigma_w^2r^2\log |S|}. 
\end{align}
\end{lemma}

\begin{proof}
Fix $\lambda>0$. By Jensen's inequality and monotonicity of the exponential it follows that
 \begin{equation}
\begin{aligned}
\label{eq:rootmetent}
\exp  \left(   \lambda \E \sup_{f\in S} M_T(f) \right) & \leq \E\exp  \left(   \lambda  \max_{f\in S} M_T(f) \right)\\
&\leq \E  \max_{f\in S} \exp  \left(   \lambda  M_T(f) \right)\\
&\leq \sum_{f\in S}  \E  \exp  \left(   \lambda  M_T(f) \right).
\end{aligned}
\end{equation}
Using $\|f\|_\infty \leq r$ and the tower property let us now estimate
\begin{align*}
\E \exp\left(\lambda M_T(f)\right)&\leq \E\exp \left(\lambda \sum_{t=1}^{T-1}4\langle w_t, f(x_t)\rangle\right)\\
&\leq \E\exp \left(\lambda \sum_{t=1}^{T-2}4\langle w_t, f(x_t)\rangle\right)\exp\left( 2\lambda^2 r^2\sigma_w^2 \right)\\
&\leq \dots\\
&\leq \exp\left( 2T \lambda^2 r^2\sigma_w^2 \right).
\end{align*}
Hence after applying logarithms to both sides of equation (\ref{eq:rootmetent}), we find
\begin{align*}
    \E \sup_{f\in S} M_T(f) \leq \frac{1}{\lambda}\left(\log |S|\right) +\lambda \left(  2T r^2\sigma_w^2 \right)
\end{align*}
which yields the result after optimizing over $\lambda>0$.
\end{proof}

\subsection{Finite Classes}
\label{sec:finiteclasses}
The generalization bound of Proposition~\ref{genpredprop} in combination with the control of the subgaussian parameter $\sigma_T^2(\mathbb{F},P_\star)$ Proposition~\ref{prop:cig} affords us, together with Lemmas~\ref{lemma:basicineq} and \ref{lemma:selfnormmetent} immediately yields Theorem~\ref{thm:finiteclassthm} below. 

\begin{theorem}
\label{thm:finitethm}
Assume that $\mathbb{F}$ has finite cardinality. Then under the assumptions of Proposition~\ref{prop:cig} it holds that
\label{thm:finiteclassthm}
\begin{align}
\label{eq:finiteclasscoreeq}
\E\| \hat f(\xi) - f(\xi)\|_2 \lesssim \sqrt{\frac{\sigma_w^2}{T}  \log (|\mathbb{F}|)}+\sqrt{\frac{M^2B^2L^2}{m^2(1-L_\star)^2T} \log (|\mathbb{F}|)}.
\end{align}

\end{theorem}

\section{Proofs Related to Stability and Learning}
\label{sec:stabproofs}

Let us now prove that contraction in an arbitrary norm $\|\cdot\|_\mathsf{X}$ implies that the subgaussian parameter $\sigma_T^2(\mathbb{F},P_\star)$ decays gracefully with time, $T$. We remind the reader that the idea is to combine a Azuma-McDiarmid-Hoeffding style of analysis with a stability argument. We now procede with this program.

\subsection{Proof of Proposition~\ref{prop:cig}}
Fix two functions $f,g \in \mathbb{F}$ and denote $h=f-g$.  Define also the function $F: \mathbb{R}^{d_x T} \to \mathbb{R}$ by
\begin{align*}
    F(\zeta_0,\dots,\zeta_{T-1})\triangleq\frac{1}{T}\sum_{t=0}^{T-1} \|h(\zeta_t)\|_2
\end{align*}
where the dummy variables $\zeta_t$ are elements of $\mathbb{R}^{d_x}$. Let also $\E_t$ denote conditional expectation with respect to $\mathcal{F}_t$ and define the Doob martingale difference sequence 
\begin{align*}
    \Delta_s \triangleq \E_s F(x_0,\dots,x_{T-1})-\E_{s-1} F(x_0,\dots,x_{T-1})
\end{align*}
with the convention $\E_{t-1} = \E$. Note now that $F(x_0,\dots,x_{T-1})-\E F(x_0,\dots,x_{T-1})=\sum_{t=0}^{T-1}\Delta_t $ so to arrive at the desired conclusion we need to prove that the $\Delta_s$ are uniformly bounded.

To this end, for a fixed $s$, we define two couplings of $(x_s,\dots,x_{T-1})$ via
\begin{align*}
    z_{t+1} &= f_\star (z_{t}), + w_t & z_s= z,& & t=s,\dots,T-1,\\
    z'_{t+1} &= f_\star (z'_{t}), + w_t & z'_s= z',& & t=s,\dots,T-1.
\end{align*}
which vary only in their initial condition $z,z'$ but are constructed with the same sequence $(w_s,\dots,w_{T-1})$. We now compute
\begin{equation}
\begin{aligned}
\label{eq:redtocont2}
    \Delta_s    &= \E_s F(x_0,\dots,x_{T-1})-\E_{s-1} F(x_0,\dots,x_{T-1})\\
                &= \E_s\frac{1}{T}\sum_{t=0}^{T-1} \|h(x_t)\|_2-\E_{s-1}\frac{1}{T}\sum_{t=0}^{T-1} \|h(x_t)\|_2\\
                &=\frac{1}{T}\E_s\sum_{t=s}^{T-1} \|h(x_t)\|_2-\frac{1}{T}\E_{s-1}\sum_{t=s}^{T-1}\|h(x_t)\|_2\\
                &\leq \sup_{z,z'} \frac{1}{T}\sum_{t=s}^{T-1} \|h(z_t)\|_2-\|h(z'_t)\|_2\\
                &\leq  \sup_{z,z'} \frac{1}{T}\sum_{t=s}^{T-1} \left|\|h(z_t)\|_2-\|h(z'_t)\|_2\right|\\
                &\leq  \sup_{z,z'} \frac{1}{T}\sum_{t=s}^{T-1} \|h(z_t)-h(z'_t)\|_2\\
                &\leq  \sup_{z,z'} \frac{2L}{T}\sum_{t=s}^{T-1} \|z_t-z'_t\|_2
\end{aligned}
\end{equation}
where the first inequality uses the Markov property to realize the conditional expectations as functions of $x_s$ and $x_{s-1}$ respectively. The other inequalities follow by application of the triangle inequality and the $2L$-Lipschitzness of $h$.

Let us now bound the $\|\cdot\|_{\mathsf{X}}$-distance between $z_t$ and $z_t'$:
\begin{equation}
\begin{aligned}
\label{eq:conttodiarmid2}
    \|z_t-z_t'\|_{\mathsf{X}}  
    &= \|f_\star(z_{t-1})+w_{t-1} -f_\star(z'_{t-1})+w_{t-1}\|_{\mathsf{X}}\\
    &\leq L_\star \|z_{t-1}-z_{t-1}'\|_{\mathsf{X}}\\
    &\dots\\
    &\leq L_\star^{t-s-1}\|z'-z\|_{\mathsf{X}}.
\end{aligned}
\end{equation}
Combining equations (\ref{eq:redtocont2}) and (\ref{eq:conttodiarmid2}), and noting that a symmetric argument applies to $-\Delta_s$ it follows that
\begin{align}
\label{eq:theboundondelta2}
    |\Delta_s| \leq    \frac{4MLB}{m(1-L_\star)T}.
\end{align}
Expressing $F$ as a telescoping sum over $\Delta_s$, we can compute its moment generating function in combination with the tower property:
\begin{align*}
&\E \exp \left( \lambda\left[ F(w_0,\dots,w_{T-1}) - \E F(w_0,\dots,w_{T-1})\right]  \right)\\
&= \E \exp \left( \lambda \sum_{j=0}^{T-1} \Delta_j\right)\\
&= \E  \exp \left( \lambda \sum_{j=0}^{T-2} \Delta_j\right)\E_{T-2}\exp \left(\lambda \Delta_{T-1} \right)\\
&\leq \E  \exp \left( \lambda \sum_{j=0}^{T-2} \Delta_j\right) \exp\left(2\lambda^2\left(  \frac{4MLB}{m(1-L_\star)T}\right)^2 \right)\\
&\leq\dots\leq  \exp\ \left( \frac{32\lambda^2M^2L^2B^2}{m^2(1-L_\star)^2T} \right)
\end{align*}
using Hoeffding's inequality to bound the conditional moment generating functions of the bounded random variables $\Delta_j$ using the inequality (\ref{eq:theboundondelta2}) (see \cite{hoeffding1963probability} or Example 2.4. in \cite{wainwright2019high}). \hfill $\blacksquare$

\subsection{Extension to Exponential Incremental Input-to-State Stability}
\label{sec:ext2ISS}
In the main text we described how contraction properties of $f_\star$ lead to bounds on the subgaussian parameter $\sigma_T^2(\mathbb{F},P_\star)$. We now show  that another control-theoretic notion of stability, known as Exponential Incremental Input-to-State Stability  (E-$\delta$ISS)  is also amenable to this analysis. The E-$\delta$ISS framework was introduced by \cite{angeli2002lyapunov}. 
Let us fix two metric spaces $(\mathsf{X},\rho_{\mathsf{X}})$ and $(\mathsf{W},\rho_{\mathsf{W}})$. A family of functions $\{G_t\}_{t\in \mathbb{Z}_\geq 0}$, $G_t: \mathsf{X} \times \mathsf{W} \to \mathsf{X}$ is $(a,b,r)$-E-$\delta$ISS if for each $T \in \mathbb{N}$, every pair of sequences $\{\eta_t\}$ and $\{\zeta_t\}$, and system of equations satisfying
\begin{align*}
    \phi_{t+1} &= G_t(\phi_t,\eta_t)  & & t=0,\dots,T-1 \\
    \psi_{t+1} &= G_t(\psi_t,\zeta_t) & & t=0,\dots,T-1
\end{align*}
with $\phi_t,\psi_t \in \mathsf{X},\eta_t,\zeta_t \in \mathsf{W}$ it holds for all $t\in [T]$ that
\begin{align}
\label{eq:EISSineq}
    \rho_{\mathsf{X}}(\phi_t,\psi_t)\leq ar^t\rho_{\mathsf{X}}(\phi_0,\psi_0)+b \sum_{k=0}^{t-1}r^{t-k-1}\rho_{\mathsf{W}}(\eta_k,\zeta_k).
\end{align}

\begin{proposition}
\label{issprop}
Fix a sequence $\{\eta_t\}$ of i.i.d. random variables, and assume that $\mathsf{W}$ is bounded, $\sup_{w,w'\in\mathsf{W}}\rho_{\mathsf{W}}(w,w')\leq B$. Suppose that $\{G_t\}_{t\in \mathbb{Z}_\geq 0}$ is $(a,b,r)$-E-$\delta$ISS and consider the process
\begin{align}
\label{eq:gennonlineards}
    x_{t+1} &= G_t(x_t,\eta_t)  & & t=0,\dots,T-1.
\end{align}
Then for every function $l: \mathsf{X} \to \mathbb{R}_{\geq 0}$, $L$-Lipschitz  with respect to $\rho_{\mathsf{X}}$:
\begin{align*}
    |l(x)-l(x')| \leq L \rho_{\mathsf{X}}(x,x'), \forall x,x'\in \mathsf{X}
\end{align*}
we have
\begin{equation*}
   \E \exp\left(\lambda \left[\frac{1}{T}\sum_{t=0}^{T-1}l(x_t)- \E\frac{1}{T}\sum_{t=0}^{T-1}l (x_t)  \right] \right)    \leq \exp\ \left( \frac{8\lambda^2L^2B^2b^2}{(1-r)^2T} \right).
\end{equation*}
\end{proposition}

By choosing $l=\|f-g\|_2$ every space $\mathbb{F}$ of $L$-Lipschitz functions is $\sigma_T^2(\mathbb{F},P_\star)$-subgaussian with respect to the process (\ref{eq:gennonlineards})  with $\sigma_T^2(\mathbb{F},P_\star)= \frac{64L^2B^2b^2}{(1-r)^2T}$. Note that the constant $M^2/m^2$ appearing in Proposition~\ref{prop:cig} can be subsumed into the constants $a$, $b$ and $L$ by appropriate rescaling of $\rho_{\mathsf{X}}$ and $\rho_{\mathsf{W}}$.

\begin{proof}
As before, the idea is to lean on an Azuma-McDiarmid-Hoeffding style of analysis, but now combined with the bound (\ref{eq:EISSineq}). Fix two functions $f,g \in \mathbb{F}$ and denote $h=f-g$. Consider the function  $F(\eta_0,\dots,\eta_{T-1}) = \frac{1}{T} \sum_{t=0}^{T-1} l(x_t)$, which becomes a function of the random sequence $\{\eta_t\}$ via (\ref{eq:gennonlineards}). We shall show that this function is $ \frac{4MLBb}{m(1-r)T}$-Lipschitz with respect to the Hamming metric. To this end, introduce a coupling of $x_t$ by defining the system $z_t = G(z_{t-1},\zeta_t)$ (and with the same initial condition) and observe that
\begin{equation}
\begin{aligned}
\label{eq:redtocont}
\left| \frac{1}{T}\sum_{t=0}^{T-1} l(x_t)-\frac{1}{T}\sum_{t=0}^{T-1} l(z_t)  \right|&\leq \frac{1}{T}\sum_{t=0}^{T-1} |l(x_t)- l(z_t)| \\
&\leq \frac{L}{T}\sum_{t=0}^{T-1} \rho_{\mathsf{X}}(x_t, z_t)
\end{aligned}
\end{equation}
by repeated application of the triangle inequality and since $l$ is $L$-Lipschitz.

Let us now bound the $\rho_{\mathsf{X}}$-distance between $x_t$ and $z_t$ under the hypothesis that $\eta_t=\zeta_t, \forall t\neq j$. Then we have using the E-$\delta$-ISS bound in equation (\ref{eq:EISSineq}) that
\begin{equation}
\label{eq:conttodiarmid}
\rho_{\mathsf{X}}(x_t,z_t)  \leq b r^{t-j-1} \rho_\mathsf{W}(\eta_j,\zeta_j).
\end{equation}
Thus, combining (\ref{eq:redtocont}) with (\ref{eq:conttodiarmid}) gives
\begin{align}
\label{eq:theboundondelta}
\left| \frac{1}{T}\sum_{t=0}^{T-1} l(x_t)-\frac{1}{T}\sum_{t=0}^{T-1} l(z_t)  \right| \leq \frac{L}{T(1-r)} \rho_{\mathsf{W}}(\eta_j,\zeta_j) \leq  \frac{LBb}{(1-r)T}
\end{align}
by boundedness of $\mathsf{W}$.

We may proceed with the analysis by defining the martingale difference sequence
\begin{align*}
\Delta_j = \E[F(\eta_0,\dots,\eta_{T-1})|\eta_0,\dots,\eta_j]-\E[F(\eta_0,\dots,\eta_{T-1})|\eta_0 \dots \eta_{j-1}]
\end{align*}
which has bounded absolute value by independence of the sequence $\{\eta_t\}$ and (\ref{eq:theboundondelta}). Observe that this allows us to express $F$ as a telescoping sum, which we can readily use to compute the moment generating function in combination with the tower property:
\begin{align*}
&\E \exp \left( \lambda\left[ F(\eta_0,\dots,\eta_{T-1}) - \E F(\eta_0,\dots,\eta_{T-1})\right]  \right)\\
&= \E \exp \left( \lambda \sum_{j=0}^{T-1} \Delta_j\right)\\
&= \E  \exp \left( \lambda \sum_{j=0}^{T-2} \Delta_j\right)\E_{T-2}\exp \left(\lambda \Delta_{T-1} \right)\\
&\leq \E  \exp \left( \lambda \sum_{j=0}^{T-2} \Delta_j\right) \exp\left(2\lambda^2 \left(  \frac{2LBb}{(1-r)T}\right)^2 \right)\\
&\leq\dots\leq  \exp\ \left( \frac{8\lambda^2L^2B^2b^2}{(1-r)^2T} \right)
\end{align*}
using Hoeffding's inequality to bound the conditional moment generating functions of the bounded random variables $\Delta_j$ using (\ref{eq:theboundondelta}) (see \cite{hoeffding1963probability} or Example 2.4. in \cite{wainwright2019high}).
\end{proof}

\subsection{Proof of Proposition~\ref{prop:mig}}

Observe that the function
\begin{align*}
F(z_0,\dots,z_{T-1}) = \frac{1}{T}\sum_{t=0}^{T-1} \| f(z_t)-g(z_t)\|_2  
\end{align*}
is $2B/T$ Hamming-Lipschitz for every fixed choice of $f,g \in \mathbb{F}$. Hence we may apply Corollary 2.10 of \cite{paulin2015concentration} to obtain a high probability tail bound. This is equivalent to the desired statement by Proposition 2.5.2 in \cite{vershynin2018high}. \hfill $\blacksquare$

\section{Proof of the Decoupling Estimate, Proposition~\ref{genpredprop}}
\label{sec:proofofgenpred}
In what follows, we  compare probability integrals under different distributions. More precisely, we wish to relate the joint distribution of the least squares estimator (\ref{df:PEM}) and the samples from the system (\ref{eq:ds}) with the product measure of their marginals. The following variational formulation of $D(P\|Q)$, due to \cite{donsker1975asymptotic}, is key:

\begin{lemma}
\label{lemma:dvlemma}
Fix two probability measures $\mathbf{P}$ and $\mathbf{Q}$ on a measure space $(\Omega,\mathcal{F})$. Then for every $\mathcal{F}$-measurable $F:\Omega \to \mathbb{R}$ such that $\int e^F d\mathbf{Q}$ is finite, it holds that
\begin{align}
\label{DV}
\int F d\mathbf{P} -\log \int e^F d\mathbf{Q} \leq D (\mathbf{P}\|\mathbf{Q}).
\end{align}
Moreover, if $ D (\mathbf{P}\|\mathbf{Q})<\infty$, then equality in (\ref{DV}) is attained at $F=\log \frac{d\mathbf{P}}{d\mathbf{Q}}$.
\end{lemma}

Equipped with Lemma~\ref{lemma:dvlemma}, and inspired by the work of \cite{russo2019much} and \cite{NIPS2017_ad71c82b}, we now turn to the proof of Proposition~\ref{genpredprop}. We remark that the first paragraph of the proof is identical to the proof of Lemma 1 in \cite{NIPS2017_ad71c82b}. As it is central to our argument, we reproduce it below.
\paragraph{Proof of Proposition~\ref{genpredprop}}
 We begin by observing that by rescaling $F$ in (\ref{DV}) by $\lambda$, we obtain
\begin{align}
\label{lambdaDV}
\int \lambda F d\mathbf{P} \leq \log \int e^{\lambda F} dQ + D (\mathbf{P}\|\mathbf{Q}).
\end{align}
For any $F$ which is $\sigma^2$-subgaussian under $\mathbf{Q}$, we have that
\begin{align}
\label{subGDV}
\log \int e^{\lambda F} d\mathbf{Q} \leq \int \lambda F d\mathbf{Q} + \frac{\lambda^2 \sigma^2}{2}.
\end{align}
Combining inequalities (\ref{lambdaDV}) and (\ref{subGDV}), we see that
\begin{align*}
\int \lambda F d\mathbf{P} - \int\lambda F d\mathbf{Q} \leq D (\mathbf{P}\|\mathbf{Q}) +  \frac{\lambda^2 \sigma^2}{2}
\end{align*}
which after choice of $\lambda = -\frac{\sqrt{D(P\|Q)}}{\sqrt{2\sigma^2}}$ and rearranging becomes 
\begin{align}
\label{XRDV}
\int Fd\mathbf{Q} \leq \int F d\mathbf{P} + \sqrt{2\sigma^2 D(\mathbf{P}\|\mathbf{Q})}.
\end{align}

We now specialize this known result to our setting. Let us now choose
\begin{align*}
\mathbf{P}&=P_{Z, (f,g)}, & \mathbf{Q}&=P_{Z}\otimes P_{(f,g)},\\
F&=\frac{1}{T} \sum_{t=0}^{T-1} \|f(x_t)-g(x_t) \|_2.
\end{align*}
Observe that for $\mathbf{P},\mathbf{Q}$ as above, $D(\mathbf{P}\|\mathbf{Q}) = I(( f,g) ; Z)$. Let further $(x'_0,\dots,x'_{T-1})$ be equal in distribution to $(x_0,\dots,x_{T-1})$ but independent from $f$ and $g$. In other words $(x_0,\dots,x_{T-1},f,g)$ is drawn from $\mathbf{P}$ and $(x'_0,\dots,x'_{T-1},f,g)$ is drawn from $\mathbf{Q}$. Let also $\tau$ be uniformly distributed over $\{0,\dots,T-1\}$ and independent of all other randomness so that we may take $\xi = x'_\tau$. Hence, for these choices, inequality (\ref{XRDV}) combined with Jensen's inequality  yield
\begin{align*}
    \E \|f(\xi) - g(\xi) \|_2 &=\frac{1}{T} \sum_{t=0}^{T-1}\E \| f(x'_t) - g(x'_t) \|_2\\
    &\overset{\mathrm{(i)}}{\leq}\frac{1}{T} \sum_{t=0}^{T-1}\E \| f(x_t) - g(x_t) \|_2+\sqrt{2\sigma_T^2(\mathbb{F},P_\star) I(f,g; Z)}\\
    &=\E \| f(x_\tau) - g(x_\tau) \|_2+\sqrt{2\sigma_T^2(\mathbb{F},P_\star) I(f,g; Z)}\\
    &= \sqrt{\left(\E \| f(x_\tau) - g(x_\tau) \|_2\right)^2}+\sqrt{2\sigma_T^2(\mathbb{F},P_\star) I(f,g; Z)}\\
    &\overset{\mathrm{(ii)}}{\leq} \sqrt{\E \| f(x_\tau) - g(x_\tau) \|_2^2}+\sqrt{2\sigma_T^2(\mathbb{F},P_\star) I(f,g; Z)}\\
    &=\sqrt{\frac{1}{T} \sum_{t=0}^{T-1}\E \| f(x_t) - g(x_t) \|^2_2}+\sqrt{2\sigma_T^2(\mathbb{F},P_\star) I(f,g; Z)}
\end{align*}
by linearity of expectation and reformulating the mixture component. Inequality (i) follows from inequality (\ref{XRDV}) and inequality (ii) from Jensen's inequality.  \hfill $\blacksquare$


\subsection{Extension: Generalization Bounds for Dynamical Systems}
\label{subsec:genbounds}
It has previously been observed in the context of the generalized linear model (\ref{gls}) that system-theoretic notions are useful to provide learning guarantees, see Section 4 of \cite{foster2020learning} for an interesting discussion. Here, we show that the bounds on $\sigma_T^2(\mathbb{F},P_\star)$ in Propositions~\ref{prop:cig} and \ref{issprop} yield generalization bounds for more general statistical learning. Consider a loss function $l :\mathsf{X} \times \mathsf{Y} \times \mathsf{H} \to \mathbb{R}_{\geq 0}$ and assume that the sequences $(x_0,\dots,x_{T-1})$ and $(y_0,\dots,y_{T-1})$ are generated by E-$\delta$ISS systems (\ref{eq:gennonlineards}), $\{G^x_t\}$ and $\{G^y_t\}$ respectively. Assume that these  are driven by the same i.i.d. noise sequence $(w_0,\dots,w_{T-1})$. If not, we we can always define such a sequence on a space of the form $\mathsf{W}=\mathsf{W}_{\mathsf{X}} \times \mathsf{W}_{\mathsf{Y}}$.

The problem of statistical learning is  to find a hypothesis $h\in \mathsf{H}$ that minimizes 
\begin{align*}
   \E_Z L(Z,h) = \E_Z \frac{1}{T}\sum_{t=0}^{T-1}  l(x_{t+1},x_t,h)
\end{align*}
with $Z=(x_0,\dots,x_{T-1},y_0,\dots,y_{T-1})$ and where $\E_Z$ denotes integration over the randomness in $Z$. Let $H$ be a randomized learning algorithm (a random, data-dependent element of $\mathsf{H}$). We define its generalization error by
\begin{align*}
    \mathsf{gen}(H) = \E [\E_{\bar Z} L(\bar Z,H)- L(Z,H)]
\end{align*}
where $\bar Z$ is equal to $Z$ in distribution but independent of $H$. By combining Lemma~1 of \cite{NIPS2017_ad71c82b} with Proposition~\ref{issprop} we arrive at the following inequality.

\begin{proposition}
Suppose that $l$ is $L$-lipschitz in its first two arguments:
\begin{align*}l|(x,y,h)-l(x',y',h)| \leq  L \rho_{\mathsf{X}}(x,x')+L\rho_{\mathsf{Y}}(y,y'), \forall x,x'\in \mathsf{X},y,y' \in \mathsf{Y}, h\in \mathbb{H},
\end{align*}
that $(x_0,\dots,x_{T-1})$ is $ (a,b,r)$ E-$\delta$ISS and that $(y_0,\dots,y_{T-1})$ is $ (a',b',r')$ E-$\delta$ISS. Then 
\begin{align*}
    |\mathsf{gen}(H)| \leq \sqrt{\frac{64 L^2 B^2 b^2}{(1-r_{\max})^2 T} I(Z;H)}
\end{align*}
where $r_{\max} = \max (r,r')$ and $B = \sup_{w,w'\in\mathsf{W}}\rho_{\mathsf{W}}(w,w')$.
\end{proposition}

In principle a direct proof using the methods from Appendix~\ref{sec:stabproofs} is possible. For brevity, we instead show how the result can be reduced to the statement of Proposition~\ref{issprop}.

\begin{proof}
Let $\mathsf{X}' = \mathsf{X}\times \mathsf{Y}$ and define the extended dynamics $\phi^1_t =x_t, \phi^2_t=y_{t}$. Then
\begin{align}
\label{eq:augsys}
    \begin{bmatrix}
    \phi_{t+1}^1\\
    \phi_{t+1}^2
    \end{bmatrix} = 
    \begin{bmatrix}G^x_t(\phi^1_t,w_t) \\ G^y_t(\phi^2_t,w_t)  \end{bmatrix},
\end{align}
or $\phi_{t+1} = G_t(\phi_t,w_t)$ in brief. Since $G^x$ and $G^y$ are both E-$\delta$ISS as system from $(\mathsf{W},\rho_{\mathsf{W}})$  to $(\mathsf{X},\rho_{\mathsf{X}})$ and $(\mathsf{Y},\rho_{\mathsf{Y}})$  respectively, it follows that $G$ is $(\max(a,a'),2b,\max(r,r'))$  E-$\delta$ISS from  $(\mathsf{W},\rho_{\mathsf{W}})$ to $(\mathsf{X'},\rho_{\mathsf{X'}})$ with $\rho_{\mathsf{X}'} =\rho_{\mathsf{X}} + \rho_{\mathsf{Y}} $. Hence, we may apply Proposition~\ref{issprop} to conclude that 
\begin{align*}
    L(Z,h) =  \frac{1}{T}\sum_{t=0}^{T-1}  l(y_{t},x_t,h)
\end{align*}
is $\frac{32L^2B^2b^2}{(1-r_{\max})^2T}$-subgaussian for each fixed $h$ where $r_{\max} = \max(r,r')$. The result follows by applying Lemma~1 of \cite{NIPS2017_ad71c82b}. 
\end{proof}

\section{Supporting Material for the Examples in Section~\ref{sec:appsec}}

\subsection{Metric Entropy Calculations for Reproducing Kernel Hilbert Spaces}
\label{sec:metentcalcrkhs}
In this section we compute the metric entropy of the unit ball of a Reproducing Kernel Hilbert Space (RKHS) of real-valued functions $f : \mathsf{X} \to \mathbb{R}$  subject to an eigenvalue decay condition.

Let us recall some facts about Reproducing Kernel Hilbert Spaces and their embeddings into $L^2(\mathbf{P})$. Assume that $\mathsf{X}$ is a compact subset of $\mathbb{R}^{d_x}$ and let $K :\mathsf{X} \times \mathsf{X}\to \mathbb{R}$ be a continuous positive semidefinite kernel function. Suppose further that the Hilbert-Schmidt norm of $K$ with respect to the probability measure $\mathbf{P}$ on $\mathsf{X}$ is finite: $\int \int K^2(x,z) d\mathbf{P}(x)d\mathbf{P}(z) < \infty$. A consequence of Mercer's Theorem (Theorem 12.20 and Corollary 12.26 of \cite{wainwright2019high}) is that there exists an orthonormal basis of $\{\phi_i\}_{i=1}^\infty$ of $L^2(\mathbf{P})$ and a sequence of nonnegative real numbers $\{\lambda_i\}_{i=1}^\infty$ such that $K(x,z) = \sum_{i=1}^\infty\lambda_i \phi_i(x) \phi_i(z)$. Moreover, the RKHS associated to $K$ is given by
\begin{align*}
\mathbb{H} = \left\{ f = \sum_{i=1}^\infty b_i \phi_i \Big| \{b_i\} \subset l^2(\mathbb{N}), \sum_{i=1}^\infty \frac{b_i^2}{\lambda_i} < \infty \right\}
\end{align*}
with the inner product $\displaystyle\langle f,g \rangle_{\mathbb{H}}  = \sum_{i=1}^\infty \frac{\langle f,\phi_i \rangle \langle g,\phi_i\rangle }{\lambda_i}$, where $\langle\cdot , \cdot \rangle$ is the standard inner product in $L^2(\mathbf{P})$. The unit ball in $\mathbb{H}$ is therefore given by
\begin{align*}
\mathbb{B}_\mathbb{H}=\left\{ f = \sum_{i=1}^\infty b_i \phi_i \Big| \{b_i\} \subset l^2(\mathbb{N}), \sum_{i=1}^\infty \frac{b_i^2}{\lambda_i} \leq  1 \right\}.
\end{align*}
With this background established, we are now poised to compute the metric entropy of $\mathbb{B}$.

\begin{proposition}
Let $\mathsf{X}$ be a compact subset of $\mathbb{R}^{d_x}$ and $K$ be a continuous positive semidefinite kernel function. Assume that $\mathbb{H}$ is a RKHS generated by the kernel $K$, which further satisfies the eigenvalue decay condition $\lambda_j \lesssim j^{-2\alpha}$, for some $\alpha >0$. Assume further that the eigenfunctions $\phi_j$ of $K$ are uniformly bounded; $\sup_{x\in \mathsf{X}} |\phi_j(x)| \leq A$ for all $j =1,2,\dots,$. Then 
\begin{align*}
\log \mathcal{N} (\mathbb{B}_\mathbb{H}, \|\cdot\|_{\infty},\delta) \lesssim A^{\alpha} {\e^{-1/\alpha}}  \log \left(1+ \lambda_1 \frac{A^{1+1/\alpha}}{\delta^{1+1/2\alpha}}\right).
\end{align*}
\end{proposition}

Observe that $\mathbb{B}_\mathbb{H}$ is an ellipsoid in $L^2$, which is essentially ill-conditioned due to the eigenvalue decay condition. We shall show that it suffices to construct a covering for a finite-dimensional section of this ellipsoid corresponding to the large eigenvalues of the kernel $K$. In other words, at scale $\delta$ the ellipsoid $\mathbb{B}_\mathbb{H}$ ``looks'' finite-dimensional. The proof determines this critical dimension for a given $\delta>0$.

\begin{proof}
We may assume that the eigenvalues $\{\lambda_i\}$ are ordered as $\lambda_1 \geq \lambda _2 \geq \dots$. Fix an integer $m$ and define 
\begin{align*}\mathbb{B}_m =\left\{ g = \sum_{i=1}^mb_i \phi_i \Big| \sum_{i=1}^m \frac{b_i^2}{\lambda_i} \leq  1 \right\}. 
\end{align*}
Observe that for every $f \in \mathbb{B}_\mathbb{H}$ there exists $g \in \mathbb{B}_m$ such that $\|f-g\|_{L^\infty} \leq A \sqrt{\lambda_{m+1}}$.

 Observe now that for every $g \in \mathbb{B}_m$, $g = \sum_{i=1}^m b_i \phi_i(\cdot), b=(b_1,\dots,b_m)$  we have
\begin{align}
\label{eq:linftylone}
\| g \|_{L^\infty} = \left\| \sum_{i=1}^m b_i \phi_i(\cdot)  \right\|_{L^\infty} \leq A \| b \|_{l^1(\mathbb{R}^m)}.
\end{align}
Using this, we obtain a covering of $(\mathbb(B)_m, \|\cdot\|_\infty)$ by regarding it as a subset of $\mathbb{R}^m$. Namely, choose $N \in \mathbb{N}$ so that $\{b^1,\dots,b^N\}$ is an optimal $(\delta/A)$-covering of 
\begin{align*}
B_m = \left\{ b \in \mathbb{R}^m \Big|  \sum_{i=1}^m \frac{b_i^2}{\lambda_i} \leq  1\right\}
\end{align*}
in the metric of $l^1(\mathbb{R}^m)$ and extend it to a $\delta$-covering of $\mathbb{B}_m$ in supremum norm by introducing $\{ (b^1)^\top \phi(\cdot), \dots, (b^N)^\top \phi(\cdot) \}$ where $\phi(\cdot) = (\phi_1 (\cdot),\dots,\phi_m(\cdot))$ and using (\ref{eq:linftylone}). Now, the finite-dimensional norms are all equivalent and in particular we have that$\| \cdot \|_{l^1 (\mathbb{R}^m)} \leq \sqrt{m}\| \cdot \|_{l^2 (\mathbb{R}^m)}$. Hence, by rescaling apropriately, we require at most  $\left(1+\frac{2 A  \lambda_1\sqrt{m}}{\delta}\right)^m$ points to cover $B_m$ in $l^1$-metric, which we may thus take as an upper bound for $N$.

By hypothesis that $\lambda_j \lesssim j^{-2\alpha}$ it suffices to take $m \asymp (\delta/A)^{-1/\alpha}$ for the above covering to also cover the entirety of $\mathbb{B}_\mathbb{H}$ in $L^\infty$ since then every point of $\mathbb{B}$ is at most distance $\delta$ removed from a point of $\mathbb{B}_m$, which in turn is at most $\delta$ removed from the covering.   It follows that
\begin{align*}
\log \mathcal{N} (\mathbb{B}_\mathbb{H},  \|\cdot\|_{\infty},\delta) &\lesssim \log \left(1+\frac{A\lambda_1 \sqrt{m}}{\delta}\right)^{m} \\
&\lesssim{(\delta/A)^{-1/\alpha}}  \log \left(1+ \lambda_1 \frac{A^{1+1/2\alpha}}{\delta^{1+1/2\alpha}}\right)
\end{align*}
which we sought to prove.
\end{proof}

\subsection{An Experiment Supporting Example~\ref{ex:rkhs}}
\label{sec:experiments}

To empirically verify our claim regarding the rate of convergence of the LSE (\ref{df:PEM}) for Example~\ref{ex:rkhs} we simulate data from an autoregressive system (\ref{eq:ards}) with  $f_\star$ belonging to the RKHS with radial basis function kernel $K(x,z) = \exp\left(-\frac{1}{2}\|x-z\|_2^2\right)$. More precisely, we generate a random $f_\star$ with order $k = 10000$ and state dimension $d_x=5$ by first generating $\eta \in \mathbb{R}^{d_x \times k}$ with entries $\eta_{ij}$, drawn i.i.d. from a standard normal distribution and  $\Theta\in \mathbb{R}^{k \times d_x}$ with $\Theta_{ij}$ drawn i.i.d., also from a standard normal distribution. We then set $\bar{K}(\eta, \cdot) = (K(\eta)_i, \cdot)_{i=1}^k$, and $\tilde\Theta = \rho \Theta / |\Theta|$, for $\rho > 0$ a parameter used to control the Lipschitz constant of $f_\star$ and where $|\cdot|$ denotes the matrix operator norm. Finally, we  choose $f_\star(\cdot) = \tilde\Theta^\top \bar{K}(\eta, \cdot)$. Note that since $K(x,z)$ is $1$-Lipschitz in either argument, $f_\star$ is guaranteed to be $\rho$-contractive if $\rho<1$.

We then use $f_\star$ to generate training trajectories of varying length to be used in the LSE (\ref{df:PEM}), as well as use $f_\star$ to generate $500$ i.i.d. draws from the stationary distribution\footnote{Approximated by running the system for a burn in time of $T=1000$ time-steps before sampling from it.}. To implement the LSE (\ref{df:PEM}) we pass by the dual problem, kernel ridge regression, to estimate $\hat f$. We then approximate the $2$-norm distance $\|\hat f(\xi)-f_\star(\xi)\|_2$ by drawing $1000$ fresh trajectories of length and averaging over the final sample. We average our results over $10$ independent systems (random draws of $f_\star$) and plot our experiment in Figure~\ref{plot:niksplots}. It is interesting to note that the slope of the logarithmic plot is slightly less steep than $-1/2$. This is consistent with the near parametric rate of convergence suggested by Example~\ref{ex:rkhs} and the exponential eigenvalue decay of the kernel $K(x,z) = \exp\left(-\frac{1}{2}\|x-z\|_2^2\right)$, see \cite{wainwright2019high}, page 399.

\begin{figure}
\centering
\begin{minipage}{.5\textwidth}
  \centering
  \includegraphics[scale=0.5]{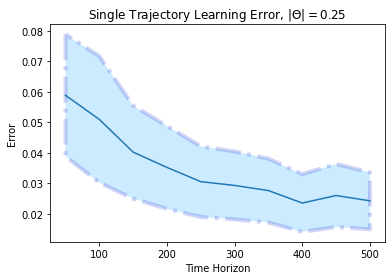}
  \includegraphics[scale=0.5]{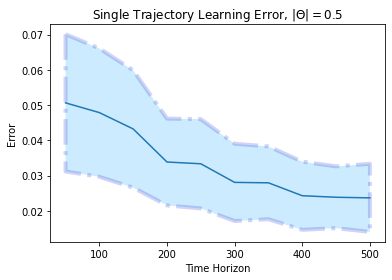}
  \includegraphics[scale=0.5]{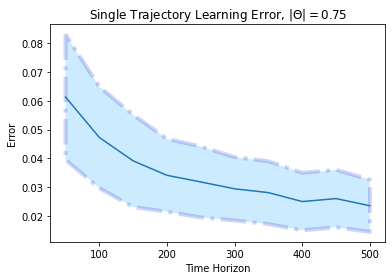}
  \label{fig:test1}
\end{minipage}%
\begin{minipage}{.5\textwidth}
  \centering
  \includegraphics[scale=0.5]{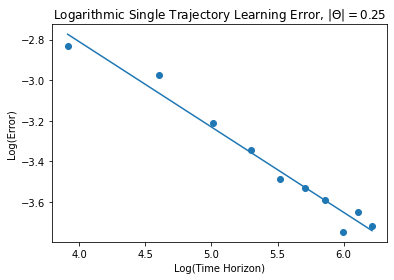}
  \includegraphics[scale=0.5]{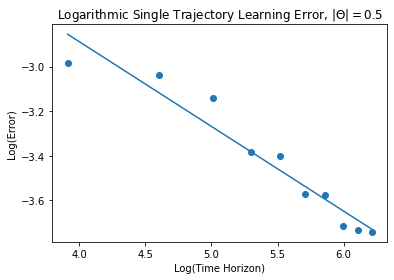}
  \includegraphics[scale=0.5]{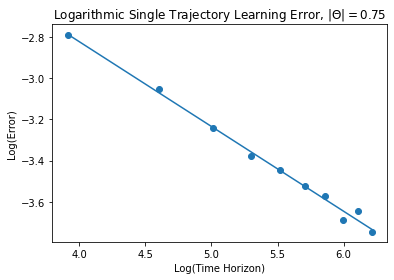}
  \label{fig:test2}
\end{minipage}
\caption{Convergence of the LSE (\ref{df:PEM}) in terms of the error $ \| \hat f(\xi) -f_\star(\xi)\|_2$ using data from a single trajectory and with a time horizon $T=500$. The plots on the left illustrate the convergence of the LSE with error bars and the right plots show this on a logarithmic scale. Notice that the slope of the line any of the rightmost plots is slightly less steep than $-1/2$.}
\label{plot:niksplots}
\end{figure}

\subsection{Proof of the claim in Example~\ref{ex:gls}}
\label{sec:ex:gls}
To use Theorem~\ref{thm:parametricrates} we need to bound the covering number of $\mathbb{F}^\phi$. Define
\begin{align*}
\mathbb{M}_{d_x,d_x} = \{A \in \mathbb{R}^{d_x\times d_x},  \| A \|_F \leq C \}.
\end{align*}
Then it is well-known that $\log \mathcal{N}(\mathbb{M}_{d_x,d_x},   \| \cdot \|_F , \delta) \leq  2  d_x \log (1+2C/\delta)$. Let now $\{A_1,\dots A_N\}$ be an optimal $\delta$-cover of $\mathbb{M}_{d_x,d_x}$. Then for every $A \in \mathbb{M}_{d_x,d_x}$  we can find $A_i \in \{A_1,\dots A_N\}$ such that
\begin{align*}
\| \phi (A \: \cdot \:) -\phi (A_i \: \cdot \:)\|_\infty \leq  \sup_{x\in X} \|(A - A_i)x\|_2 \leq  \|A - A_i\|_F \sup_{x\in X} \| x\|_2 \leq  \delta B.
\end{align*}
Hence any $\delta$-covering of $\mathbb{M}_{d_x,d_x}$ induces a $B\delta$-covering of $\mathbb{F}^\phi$ and we we have established the upper bound
\begin{equation*}
\begin{aligned}
 \log \mathcal{N}(\mathbb{F}^\phi, \|\cdot\|_\infty, \delta) &\leq \log \mathcal{N}(\mathbb{M}_{d_x,d_x},   \| \cdot \|_F , \delta/BC)\\
 &\leq  2  d_x^2 \log (1+2BC/\delta).
\end{aligned}
\end{equation*}
By Theorem~\ref{thm:parametricrates} we thus have
\begin{align*}
\label{eq:glscoreq}
\E\| \hat f(\xi) - f(\xi)\|_2 \lesssim \sqrt{\frac{\sigma_w^2}{T}  d_x^2 \log (1+2\sigma_w \sqrt{d_x} BCT^2)}+\sqrt{\sigma^2_T(\mathbb{F}^\phi)  d_x^2 \log (1+2BCT)}.
\end{align*}
so that the result follows by using Proposition~\ref{prop:cig} to bound $\sigma^2_T(\mathbb{F}^\phi, P_\star)$. \hfill $\blacksquare$

\bibliographystyle{plainnat}
\bibliography{main}

\end{document}